\documentclass{comsoc2016}

\title{Complexity Results for Manipulation,\\Bribery and Control
of the Kemeny\\Procedure in Judgment Aggregation}
\author{Ronald de Haan\\[5pt]{\small Algorithms \& Complexity Group\\Technische Universit\"{a}t Wien}}

\usepackage{hyperref}

\usepackage{mathtools}
\usepackage{xcolor}

\usepackage{booktabs}
\usepackage{multirow}

\newcommand{\SB}{\{\,}%
\newcommand{\SM}{\;{:}\;}%
\newcommand{\SE}{\,\}}%
\newcommand{\SBs}{\{}%
\newcommand{\SEs}{\}}%

\renewcommand{\P}{\text{\normalfont P}}
\newcommand{\NP}{\text{\normalfont NP}}

\newcommand{\co}{\text{\normalfont co-}}

\newcommand{\BB}{\mathbb{B}}
\newcommand{\NN}{\mathbb{N}}

\newcommand{\III}{\mathcal{I}}
\newcommand{\JJJ}{\mathcal{J}}

 \newcommand{\RRR}{\mathcal{R}}

\newcommand{\Card}[1]{|#1|}
\newcommand{\size}[1]{\ensuremath{\mtext{size}(#1)}}

\newcommand{\mtext}[1]{\text{\normalfont #1}}

\newcommand{\QSat}[1]{\ensuremath{\mtext{\sc QSat}_{#1}}}

\newcommand{\Var}[1]{\mtext{Var(\ensuremath{#1})}}

\newcommand{\Pre}[1]{\mtext{[\ensuremath{#1}]}}

\newcommand{\bfSigmaP}[1]{\ensuremath{\mathbf{\Sigma}^{\mtext{\textbf{p}}}_{\textbf{#1}}}}

\newcommand{\SigmaP}[1]{\ensuremath{\Sigma^{\mtext{p}}_{#1}}}
\newcommand{\PiP}[1]{\ensuremath{\Pi^{\mtext{p}}_{#1}}}

\newcommand{\prof}[1]{\text{\boldmath $#1$}}

\newcommand{\Dist}[0]{\ensuremath{\mtext{Dist}}}

\newcommand{\Kemeny}[0]{\mtext{Kemeny}}

\newcommand{\CautiousManipulation}[1]{\ensuremath{\mtext{\textsc{Cautious}-}\allowbreak\mtext{\textsc{Mani}}\-\mtext{\textsc{pulation}}\-\mtext{(#1)}}}
\newcommand{\OptimisticManipulation}[1]{\ensuremath{\mtext{\textsc{Optimistic}-}\allowbreak\mtext{\textsc{Mani}}\-\mtext{\textsc{pulation}}\-\mtext{(#1)}}}
\newcommand{\PessimisticManipulation}[1]{\ensuremath{\mtext{\textsc{Pessimistic}-}\allowbreak\mtext{\textsc{Mani}}\-\mtext{\textsc{pulation}}\-\mtext{(#1)}}}
\newcommand{\SuperoptimisticManipulation}[1]{\ensuremath{\mtext{\textsc{Superoptimistic}-}\allowbreak\mtext{\textsc{Mani}}\-\mtext{\textsc{pulation}}\-\mtext{(#1)}}}
\newcommand{\SafeManipulation}[1]{\ensuremath{\mtext{\textsc{Safe}-}\allowbreak\mtext{\textsc{Super}}\-\mtext{\textsc{optimistic}-}\allowbreak\mtext{\textsc{Mani}}\-\mtext{\textsc{pulation}}\-\mtext{(#1)}}}

\newcommand{\CautiousExactManipulation}[1]{\ensuremath{\mtext{\textsc{Cautious}-}\allowbreak\mtext{\textsc{Exact}-}\allowbreak\mtext{\textsc{Manipulation}}\-\mtext{(#1)}}}
\newcommand{\BraveExactManipulation}[1]{\ensuremath{\mtext{\textsc{Brave}-}\allowbreak\mtext{\textsc{Exact}-}\allowbreak\mtext{\textsc{Manipulation}}\-\mtext{(#1)}}}

\newcommand{\CautiousBribery}[1]{\ensuremath{\mtext{\textsc{Cautious}-}\allowbreak\mtext{\textsc{Bribery}}\-\mtext{(#1)}}}
\newcommand{\OptimisticBribery}[1]{\ensuremath{\mtext{\textsc{Optimistic}-}\allowbreak\mtext{\textsc{Bribery}}\-\mtext{(#1)}}}
\newcommand{\PessimisticBribery}[1]{\ensuremath{\mtext{\textsc{Pessimistic}-}\allowbreak\mtext{\textsc{Bribery}}\-\mtext{(#1)}}}
\newcommand{\SuperoptimisticBribery}[1]{\ensuremath{\mtext{\textsc{Superoptimistic}-}\allowbreak\mtext{\textsc{Bribery}}\-\mtext{(#1)}}}
\newcommand{\SafeBribery}[1]{\ensuremath{\mtext{\textsc{Safe}-}\allowbreak\mtext{\textsc{Super}}\-\mtext{\textsc{optimistic}-}\allowbreak\mtext{\textsc{Bribery}}\-\mtext{(#1)}}}

\newcommand{\CautiousExactBribery}[1]{\ensuremath{\mtext{\textsc{Cautious}-}\allowbreak\mtext{\textsc{Exact}-}\allowbreak\mtext{\textsc{Bribery}}\-\mtext{(#1)}}}
\newcommand{\BraveExactBribery}[1]{\ensuremath{\mtext{\textsc{Brave}-}\allowbreak\mtext{\textsc{Exact}-}\allowbreak\mtext{\textsc{Bribery}}\-\mtext{(#1)}}}

\newcommand{\CautiousExactControlByAddingIssues}[1]{\ensuremath{\mtext{\textsc{Cautious}-}\allowbreak\mtext{\textsc{Exact}-}\allowbreak\mtext{\textsc{Control}-}\allowbreak\mtext{\textsc{by}-}\allowbreak\mtext{\textsc{Adding}-}\allowbreak\mtext{\textsc{Issues}}\-\mtext{(#1)}}}
\newcommand{\BraveExactControlByAddingIssues}[1]{\ensuremath{\mtext{\textsc{Brave}-}\allowbreak\mtext{\textsc{Exact}-}\allowbreak\mtext{\textsc{Control}-}\allowbreak\mtext{\textsc{by}-}\allowbreak\mtext{\textsc{Adding}-}\allowbreak\mtext{\textsc{Issues}}\-\mtext{(#1)}}}

\newcommand{\CautiousExactControlByDeletingIssues}[1]{\ensuremath{\mtext{\textsc{Cautious}-}\allowbreak\mtext{\textsc{Exact}-}\allowbreak\mtext{\textsc{Control}-}\allowbreak\mtext{\textsc{by}-}\allowbreak\mtext{\textsc{Deleting}-}\allowbreak\mtext{\textsc{Issues}}\-\mtext{(#1)}}}
\newcommand{\BraveExactControlByDeletingIssues}[1]{\ensuremath{\mtext{\textsc{Brave}-}\allowbreak\mtext{\textsc{Exact}-}\allowbreak\mtext{\textsc{Control}-}\allowbreak\mtext{\textsc{by}-}\allowbreak\mtext{\textsc{Deleting}-}\allowbreak\mtext{\textsc{Issues}}\-\mtext{(#1)}}}

\newenvironment{myquote}{\begin{center}
    \begin{minipage}{.96\linewidth}}{\end{minipage}\end{center}}
    
\newcommand{\probdef}[1]{
  \begin{myquote}
    \framebox[\linewidth-15pt][l]{\parbox{\linewidth-25pt}{
    #1
    }}
  \end{myquote}
}

\DeclareRobustCommand{\DE}[3]{#2}

\usepackage{amsthm}
\usepackage{amssymb}

\newtheorem{theorem}{Theorem}

\newtheorem{proposition}[theorem]{Proposition}
\newtheorem{proposition*}[theorem]{Proposition$^{\star}$}

\newtheorem{corollary}[theorem]{Corollary}
\newtheorem{corollary*}[theorem]{Corollary$^{\star}$}

\allowdisplaybreaks

\begin{document}


\begin{abstract}
  We study the computational complexity of several scenarios of strategic
  behavior for the Kemeny procedure in the setting of judgment aggregation.
  In particular, we investigate (1)~manipulation, where an individual aims to
  achieve a better group outcome by reporting an insincere individual opinion,
  (2)~bribery, where an external agent aims to achieve an outcome with
  certain properties by bribing a number of individuals,
  and (3)~control (by adding or deleting issues),
  where an external agent aims to achieve an outcome
  with certain properties by influencing the set of issues in the judgment
  aggregation situation.
  We show that determining whether these types of strategic behavior
  are possible (and if so, computing a policy for successful strategic behavior)
  is complete for the second level of the Polynomial Hierarchy.
  That is, we show that these problems are \SigmaP{2}-complete.
\end{abstract}



\section{Introduction}
\label{sec:introduction}

An important topic in the research field of computational social choice
is the (im)possibility of strategic behavior in collective decision making.
This is epitomized by the eminence of results such as the
Gibbard-Satterthwaite Theorem \cite{Gibbard73,Satterthwaite75},
that identifies various conditions under which
strategic voting (or manipulation) is, in principle, unavoidable.
Manipulation in voting is a typical example of strategic behavior, and involves
individuals reporting insincere preferences with the aim of obtaining
a group outcome that is preferable for them.

Since strategic behavior in collective decision making is generally
considered to be (socially) undesirable, a lot of research effort has been
invested in diagnosing what social choice procedures are resistant
to strategic behavior, and under what conditions.
An important research direction along these lines investigates how
computational complexity can be used to establish that various
social choice procedures are (in many cases) practically immune
to strategic behavior \cite{BartholdiToveyTrick89,ConitzerWalsh16}.
For example, in many cases, it is in principle possible to manipulate
voting rules (by reporting insincere preferences), but determining what
insincere preference leads to a better outcome is computationally so
demanding that it prevents manipulative behavior from being a useful
policy.

\paragraph{Contributions}
In this paper, we use the framework of computational complexity theory
to study several scenarios of strategic behavior in the setting of judgment
aggregation.
Judgment aggregation studies collective decision making on a set of
issues that are logically related \cite{Endriss16}.
In particular, we study three scenarios of strategic behavior
for the \emph{Kemeny judgment aggregation procedure}---which is one of the
most prominent judgment aggregation procedures known from the
literature.
We investigate:
\begin{enumerate}
  \item \emph{manipulation}, where an individual reports an insincere
    individual judgment in an attempt to enforce a preferable
    group judgment (from their point of view);
  \item \emph{bribery}, where an external party bribes several
    individuals that are involved in the group decision process
    (that is, the briber stipulates their individual judgments)
    in order to obtain a group judgment with certain properties; and
  \item \emph{control}, where an external party controls the
    set of issues that are involved in the judgment aggregation
    setting, with the aim of achieving a group judgment with
    certain properties.
\end{enumerate}
Concretely, we study various different
decision problems that formalize the computational tasks involved in
the strategic behavior in each of these scenarios.
We show that all the computational problems that we consider in this
paper are \textbf{\bfSigmaP{2}-complete}.
That is, we show that:
\begin{itemize}
  \item Manipulation for the Kemeny rule in judgment aggregation
    is \SigmaP{2}-complete.
  \item Bribery for the Kemeny rule in judgment aggregation
    is \SigmaP{2}-complete.
  \item Control (by adding or removing issues)
    for the Kemeny rule in judgment aggregation
    is \SigmaP{2}-complete.
\end{itemize}
(Completeness for the complexity class \SigmaP{2} indicates that a problem
is computationally intractable. Even the easier problem of checking
whether a given candidate solution is in fact a solution is not efficiently
solvable---it requires solving an \NP{}-complete problem.)

Various different frameworks have been used in the literature
to formalize the setting of judgment aggregation
(see, e.g.,~\cite{EndrissGrandiDeHaanLang16}).
The computational complexity results that we develop hold for two
commonly considered judgment aggregation frameworks:
\emph{formula-based judgment aggregation}
and \emph{constraint-based judgment aggregation}.
We discuss these judgment aggregation frameworks in more
detail in Section~\ref{sec:ja}.
(In order to capture the scenario of control naturally
in the constraint-based judgment aggregation framework,
we consider a slightly extended variant of this framework.
For more details, see Section~\ref{sec:ja-baic}.)

Most of the various forms of strategic behavior that we consider in this paper
involve the incentive of achieving a \emph{preferable} group outcome.
There are various ways to define preference relations over (individual and
group) judgments.
The preferences that we study are based on weighted Hamming distances.
That is, we consider weight functions that assign to each issue a
weight that indicates how important it is for an individual (or
for an external party) that the group judgment agree with their judgment
on this issue.
Such weight functions naturally induce preference relations over judgments.

In addition, we study variants of the strategic behavior scenarios where
the objective is to obtain a group judgment that includes a given set
of conclusions.
This can be seen as an all-or-nothing variants (the group outcome
either includes the required set of conclusions or it does not),
whereas the variants involving preferences based on weighted Hamming
distances offer a more gradual view (maybe the optimal outcome is not
possible, but the current outcome can still be improved slightly by
behaving strategically).

\paragraph{Worst-case Complexity}
The computational intractability results that we provide in this paper
can be seen as positive results, since they show that various kinds
of undesirable strategic behavior cannot be used efficiently across the
board due to computational complexity obstructions.
However, it is important to emphasize that the computational complexity
results that we provide in this paper are \emph{worst-case complexity}
results.
Worst-case intractability results indicate that there is no
algorithm that works efficiently in all possible cases.
However, it might well be the case that there are restricted
settings where several forms of strategic behavior are efficiently possible.

In order to consolidate the conclusion that strategic behavior for the
Kemeny procedure in judgment aggregation is computationally intractable,
further research is needed.
Such further research would have to establish that the various forms of
strategic behavior remain computationally intractable in many restricted
settings.
A key tool for establishing computational complexity results for restricted
settings is the paradigm of parameterized complexity
\cite{DowneyFellows99,DowneyFellows13,FlumGrohe06,Niedermeier06}%
---this is a framework where the complexity of computational problems
is measured in a multi-dimensional way, in contrast to the classical theory
of computational complexity, where the complexity of problems is measured
only in terms of the input size in bits.

\paragraph{Related Work}
The concept of manipulation in judgment aggregation has been studied
before in the literature, both from an axiomatic point of view
\cite{BotanNovaroEndriss16,DietrichList07,DokowFalik12}
and from a computational complexity point of view
\cite{BaumeisterErdelyiErdelyiRothe15,EndrissGrandiPorello12}.
The complexity analysis of manipulation in judgment aggregation
that has been done in the literature is restricted to uniform
premise-based quota rules.
Additionally, bribery in judgment aggregation has been studied
from a computational complexity point of view
for uniform premise-based quota rules
\cite{BaumeisterErdelyiErdelyiRothe15}.

\paragraph{Outline}
We begin in Section~\ref{sec:prelims} by considering relevant notions
from computational complexity and judgment aggregation that we use
in this paper.
Then, in Section~\ref{sec:manipulation}, we develop the intractability
results for the scenario of manipulation.
In Section~\ref{sec:bribery}, we turn to the scenario of bribery,
and in Section~\ref{sec:control}, we consider the scenario of control
(by adding or deleting issues).
Finally, we conclude in Section~\ref{sec:conclusion}.

\section{Preliminaries}
\label{sec:prelims}

Before we turn to the complexity results that we develop in this paper,
we review several relevant concepts from computational complexity
theory and judgment aggregation.

\subsection{Complexity Theory}

We begin with reviewing some basic notions from computational complexity.
We assume the reader to be familiar with the complexity classes
\P{} and \NP{}, and with basic notions such as polynomial-time
reductions.
For more details, we refer to textbooks on computational complexity
theory (see, e.g.,~\cite{AroraBarak09}).

We briefly review the classes of the Polynomial Hierarchy (PH)
\cite{MeyerStockmeyer72,Papadimitriou94,Stockmeyer76,Wrathall76}.
In order to do so, we consider quantified Boolean formulas.
A \emph{(fully) quantified Boolean formula (in prenex form)} is a formula
of the form~$Q_1 x_1 Q_2 x_2 \dotsc Q_n x_n. \psi$,
where all~$x_i$ are propositional variables,
each~$Q_i$ is either an existential or a universal quantifier,
and~$\psi$ is a (quantifier-free) propositional formula over the
variables~$x_1,\dotsc,x_n$.
Truth for such formulas is defined in the usual way.

To consider the complexity classes of
the PH, we restrict the number of quantifier alternations
occurring in quantified Boolean formulas, i.e., the number of
times where~$Q_i \neq Q_{i+1}$.
We consider the complexity classes~\SigmaP{k}, for each~$k \geq 1$.
Let~$k \geq 1$ be an arbitrary, fixed constant.
The complexity class~\SigmaP{k} consists of all decision problems
for which there exists a polynomial-time reduction to the problem
\QSat{k}, that is defined as follows.
Instances of the problem \QSat{k} are quantified Boolean formulas of the
form~$\exists x_1 \dotsc \exists x_{\ell_1} \forall x_{\ell_1+1}
\dotsc \forall x_{\ell_2} \dotsc \allowbreak Q_k x_{\ell_{k-1}+1} \dotsc Q_k x_{\ell_k}.$ $\psi$,
where~$Q_k = \exists$ if~$k$ is odd and~$Q_k = \forall$ if~$k$ is even,
where~$1 \leq \ell_{1} \leq \dotsm \leq \ell_{k}$,
and where~$\psi$ is quantifier-free.
The problem is to decide if the quantified Boolean formula is true.
The complementary class~\PiP{k} consists of all decision problems
for which there exists a polynomial-time reduction to the problem \co\QSat{k},
that is complementary to the problem \QSat{k}.
The Polynomial Hierarchy (PH) consists of the classes \SigmaP{k} and \PiP{k},
for all~$k \geq 1$.

Alternatively, one can characterize the class \SigmaP{2} using nondeterministic
polynomial-time algorithms with access to an oracle for an \NP{}-complete problem.
Let~$O$ be a decision problem.
A Turing machine~$\mathbb{M}$ with access to an \emph{$O$~oracle}
is a Turing machine with a dedicated \emph{oracle tape}
and dedicated states~$q_{\mtext{query}}$,~$q_{\mtext{yes}}$
and~$q_{\mtext{no}}$.
Whenever~$\mathbb{M}$ is in the state~$q_{\mtext{query}}$,
it does not proceed
according to the transition relation, but instead it transitions into
the state~$q_{\mtext{yes}}$ if the oracle tape contains
a string~$x$ that is a yes-instance for the problem~$O$, i.e.,~if~$x \in O$,
and it transitions into the state~$q_{\mtext{no}}$ if~$x \not\in O$.
Intuitively, the oracle solves arbitrary instances of~$O$
in a single time step.
The class \SigmaP{2} consists of all decision problems that can be solved
in polynomial time by a nondeterministic Turing machine that has access
to an $O$-oracle, for some~$O \in \NP{}$.

\subsection{Judgment Aggregation}
\label{sec:ja}

Next, we introduce the two formal judgment aggregation frameworks
that we use in this paper:
\emph{formula-based judgment aggregation}
(as used by, e.g.,~\cite{DietrichList08,%
EndrissGrandiPorello12,LangSlavkovik14})
and \emph{constraint-based judgment aggregation}
(as used by, e.g.,~\cite{Grandi12}).
Moreover, we briefly discuss an extended variant of the
constraint-based judgment aggregation framework
(as considered in, e.g.,~\cite{EndrissGrandiDeHaanLang16}).

\subsubsection{Formula-Based Judgment Aggregation}

We begin with the framework of formula-based judgment aggregation.

An \emph{agenda} is a finite, nonempty set~$\Phi$ of formulas that
does not contain any doubly-negated formulas and that is closed
under complementation.
Moreover, if~$\Phi = \SBs \varphi_1,\dotsc,\varphi_n, \allowbreak{}
\neg \varphi_1,\dotsc,\neg \varphi_n \SEs$ is an agenda,
then we let~$\Pre{\Phi} = \SBs \varphi_1,\dotsc,\varphi_n \SEs$
denote the \emph{pre-agenda} associated to the agenda~$\Phi$.
We denote the bitsize of the agenda~$\Phi$
by~$\size{\Phi} = \sum\nolimits_{\varphi \in \Phi} \Card{\varphi}$.
A \emph{judgment set~$J$} for an agenda~$\Phi$ is a
subset~$J \subseteq \Phi$.
We call a judgment set~$J$ \emph{complete} if~$\varphi \in J$
or~${\sim}\varphi \in J$ for all~$\varphi \in \Phi$;
and we call it \emph{consistent} if there exists an assignment
that makes all formulas in~$J$ true.
Intuitively, the consistent and complete judgment sets are the
opinions that individuals and the group can have.

We associate with each agenda~$\Phi$ an integrity
constraint~$\Gamma$, that can be used to further restrict the
set of feasible opinions.
Such an \emph{integrity constraint} consists of
a single propositional formula.
We say that a judgment set~$J$ is \emph{$\Gamma$-consistent}
if there exists a truth assignment that simultaneously makes
all formulas in~$J$ and~$\Gamma$ true.
Let~$\JJJ(\Phi,\Gamma)$ denote the set of all complete and
$\Gamma$-consistent subsets of~$\Phi$.
We say that finite sequences~$\prof{J} \in \JJJ(\Phi,\Gamma)^{+}$
of complete and $\Gamma$-consistent judgment sets are \emph{profiles},
and where convenient we equate a profile~$\prof{J} = (J_1,\dotsc,J_p)$
with the (multi)set~$\SBs J_1,\dotsc,J_p \SEs$.
Moreover, for~$i \in [p]$, we let~$\prof{J}_{-i}$
denote the profile~$(J_1,\dotsc,J_{i-1},J_{i+1},\dotsc,J_{p})$.

A \emph{judgment aggregation procedure} (or \emph{rule})
for the agenda~$\Phi$ and the integrity constraint~$\Gamma$
is a function~$F$ that takes as input a
profile~$\prof{J} \in \JJJ(\Phi,\Gamma)^{+}$,
and that produces a non-empty set of non-empty judgment sets.
We call a judgment aggregation procedure~$F$ \emph{resolute}
if for any profile~$\prof{J}$ it returns a singleton,
i.e.,~$\Card{F(\prof{J})} = 1$;
otherwise, we call~$F$ \emph{irresolute}.
We call a judgment aggregation procedure~$F$ \emph{anonymous}
if for every profile~$\prof{J} = (J_1,\dotsc,J_p)$
and for every permutation~$\pi : [p] \rightarrow [p]$
it holds that~$F(\prof{J}) = F(\prof{J}')$,
where~$\prof{J}' = (J_{\pi(1)},\dotsc,J_{\pi(p)})$.
An example of a resolute, anonymous judgment aggregation procedure
is the \emph{strict majority rule~$\mtext{Majority}$},
where~$\mtext{Majority}(\prof{J}) = \SBs J^* \SEs$,
where~$\varphi \in J^*$ if and only
if~$\varphi$ occurs in the strict majority of judgment sets in~$\prof{J}$,
for all~$\varphi \in \Pre{\Phi}$,
and where~$\varphi \in J^*$ if and only
if~${\sim}\varphi \not\in J^*$, for all~$\varphi \in \Phi$.
We call a judgment aggregation procedure~$F$
\emph{complete} and
\emph{$\Gamma$-consistent}, if~$J$ is complete
and $\Gamma$-consistent, respectively,
for every~$\prof{J} \in \JJJ(\Phi,\Gamma)^{+}$
and every~$J \in F(\prof{J})$.
The procedure~$\mtext{Majority}$ is not consistent.
Consider the agenda~$\Phi$ with~$\Pre{\Phi} = \SBs p,q,
p \rightarrow q \SEs$, and the profile~$\prof{J} = (J_1,J_2,J_3)$,
where~$J_1 = \SBs p, q, (p \rightarrow q) \SEs$,~$J_2 =
\SBs p, \neg q, \neg (p \rightarrow q) \SEs$, and~$J_3 =
\SBs \neg p, \neg q, (p \rightarrow q) \SEs$.
The unique outcome~$\SBs p, \neg q, (p \rightarrow q) \SEs$
in~$\mtext{Majority}(\prof{J})$ is inconsistent.

The \emph{Kemeny aggregation procedure} is based on a notion of distance.
This distance is based on the Hamming distance~$d(J,J') = \Card{\SB \varphi \in \Pre{\Phi} \SM \varphi \in (J \setminus J') \cup (J' \setminus J) \SE}$
between two complete judgment sets~$J,J'$.
Intuitively, the Hamming distance~$d(J,J')$ counts the number of issues
on which two judgment sets disagree.
Let~$J$ be a single $\Gamma$-consistent and complete judgment
set, and let~$(J_1,\dotsc,J_p) = \prof{J} \in \JJJ(\Phi,\Gamma)^{+}$
be a profile.
We define the distance between~$J$ and~$\prof{J}$
to be~$\Dist(J,\prof{J}) = \sum\nolimits_{i \in [p]} d(J,J_i)$.
Then, we let the outcome~$\mtext{Kemeny}_{\Phi,\Gamma}(\prof{J})$
of the Kemeny rule be the set of those~$J^* \in \JJJ(\Phi,\Gamma)$
for which there is no~$J \in \JJJ(\Phi,\Gamma)$
such that~$\Dist(J,\prof{J}) < \Dist(J^*,\prof{J})$.
(If~$\Phi$ and~$\Gamma$ are clear from the context,
we often write~$\mtext{Kemeny}(\prof{J})$ to
denote~$\mtext{Kemeny}_{\Phi,\Gamma}(\prof{J})$.)
Intuitively, the Kemeny rule selects those complete
and $\Gamma$-consistent judgment sets that minimize the cumulative
Hamming distance to the judgment sets in the profile.
The Kemeny rule is irresolute, complete, $\Gamma$-consistent
and anonymous.

\subsubsection{Constraint-Based Judgment Aggregation}

We continue with the framework of constraint-based judgment aggregation.

Let~$\III = \SBs x_1,\dotsc,x_n \SEs$ be a finite set of \emph{issues},
in the form of propositional variables.
Intuitively, these issues are the topics about which the individuals want
to combine their judgments.
A truth assignment~$\alpha : \III \rightarrow \BB{}$ is
called a \emph{ballot}, and represents
an opinion that individuals and the group can have.
We will also denote ballots~$\alpha$ by a binary
vector~$(b_1,\dotsc,b_{n}) \in \BB{}^{n}$,
where~$b_i = \alpha(x_i)$ for each~$i \in [n]$.
Moreover, we say that~$(p_1,\dotsc,p_{n}) \in \SBs 0,1,\star \SEs^{n}$
is a \emph{partial ballot}, and that~$(p_1,\dotsc,p_{n})$ \emph{agrees
with} a ballot~$(b_1,\dotsc,b_{n})$ if~$p_i = b_i$ whenever~$p_i \neq \star$,
for all~$i \in [n]$.
As in the case for formula-based judgment aggregation,
we introduce an integrity
constraint~$\Gamma$, that can be used to restrict the
set of feasible opinions (for both the individuals and the group).
The integrity constraint~$\Gamma$ is a propositional formula
on the variables~$x_1,\dotsc,x_n$.
We define the set~$\RRR(\III,\Gamma)$ of \emph{rational ballots}
to be the ballots (for~$\III$) that satisfy the integrity
constraint~$\Gamma$.
Rational ballots in the constraint-based judgment aggregation
framework correspond to complete and $\Gamma$-consistent
judgment sets in the formula-based judgment aggregation framework.
We say that finite sequences~$\prof{r} \in \RRR(\III,\Gamma)^{+}$
of rational ballots are \emph{profiles},
and where convenient we equate a profile~$\prof{r} = (r_1,\dotsc,r_p)$
with the (multi)set~$\SBs r_1,\dotsc,r_p \SEs$.
Moreover, for~$i \in [p]$, we let~$\prof{r}_{-i}$
denote the profile~$(r_1,\dotsc,r_{i-1},r_{i+1},\dotsc,r_{p})$.

A \emph{judgment aggregation procedure} (or \emph{rule}),
for the set~$\III$ of issues and the integrity constraint~$\Gamma$,
is a function~$F$ that takes as input a
profile~$\prof{r} \in \RRR(\III,\Gamma)^{+}$,
and that produces a non-empty set of ballots.
%
We call a judgment aggregation procedure~$F$ \emph{resolute}
if for any profile~$\prof{r}$ it returns a singleton,
i.e.,~$\Card{F(\prof{r})} = 1$;
otherwise, we call~$F$ \emph{irresolute}.
We call a judgment aggregation procedure~$F$ \emph{anonymous}
if for every profile~$\prof{r} = (r_1,\dotsc,r_p)$
and for every permutation~$\pi : [p] \rightarrow [p]$
it holds that~$F(\prof{r}) = F(\prof{r}')$,
where~$\prof{r}' = (r_{\pi(1)},\dotsc,r_{\pi(p)})$.
We call a judgment aggregation procedure~$F$
\emph{rational} (or \emph{consistent}),
if~$r$ is rational
for every~$\prof{r} \in \RRR(\III,\Gamma)^{+}$
and every~$r \in F(\prof{r})$.

As an example of a judgment aggregation procedure
we consider the \emph{strict majority rule~$\mtext{Majority}$},
where~$\mtext{Majority}(\prof{r}) = \SBs (b_1,\dotsc,b_{n}) \SEs$
and where each~$b_i$ agrees with the majority of the $i$-th bits
in the ballots in~$\prof{r}$ (in case of a tie, we arbitrarily let~$b_i = 0$).
To see that~$\mtext{Majority}$ is not rational,
consider the set~$\III = \SBs x_1,x_2,x_3 \SEs$
of issues, the integrity constraint~$\Gamma = x_3 \leftrightarrow
(x_1 \rightarrow x_2)$, and the profile~$\prof{r} =
(r_1,r_2,r_3)$,
where~$r_1 = (1,1,1)$,%
~$r_2 = (1,0,0)$, and~$r_3 = (0,0,1)$.
The unique outcome~$(1,0,1)$
in~$\mtext{Majority}(\prof{r})$ is not rational.

The \emph{Kemeny aggregation procedure} is defined for the
constraint-based judgment aggregation framework as follows.
Similarly to the case for formula-based judgment aggregation,
the Kemeny rule is based on the Hamming
distance~$d(r,r') = \Card{\SB i \in [n] \SM b_i \neq b'_i \SE}$,
between two rational ballots~$r = (b_1,\dotsc,b_{n})$
and~$r' = (b'_1,\dotsc,b'_{n})$ for the set~$\III$ of issues and
the integrity constraint~$\Gamma$.
Let~$r$ be a single ballot,
and let~$(r_1,\dotsc,r_p) = \prof{r} \in \RRR(\III,\Gamma)^{+}$
be a profile.
We define the distance between~$r$ and~$\prof{r}$
to be~$\Dist(r,\prof{r}) = \sum\nolimits_{i \in [p]} d(r,r_i)$.
Then, we let the outcome~$\mtext{Kemeny}_{\III,\Gamma}(\prof{r})$
of the Kemeny rule be the set of those ballots~$r^* \in \RRR(\III,\Gamma)$
for which there is no~$r \in \RRR(\III,\Gamma)$
such that~$\Dist(r,\prof{r}) < \Dist(r^*,\prof{r})$.
(If~$\III$ and~$\Gamma$ are clear from the context,
we often write~$\mtext{Kemeny}(\prof{r})$ to
denote~$\mtext{Kemeny}_{\III,\Gamma}(\prof{r})$.)
The Kemeny rule is irresolute, anonymous and rational.

\subsubsection{Extended Constraint-Based Judgment Aggregation}
\label{sec:ja-baic}

Finally, we consider an extended variant of the constraint-based judgment aggregation
framework.
In the constraint-based judgment aggregation framework,
we consider a set~$\III$ of issues and an integrity constraint in the form
of a propositional formula~$\Gamma$ that satisfies the constraint
that~$\Var{\Gamma} \subseteq \III$.
However, in some situations it is more convenient to allow the integrity
constraint~$\Gamma$ to contain additional variables.
In the extended constraint-based framework, we relax the condition
that~$\Var{\Gamma} \subseteq \III$, and we allow arbitrary propositional formulas
as integrity constraints.
This modification requires us to adapt the notion of rationality accordingly.
A ballot~$\alpha : \III \rightarrow \BB{}$ is said to be \emph{rational}
if~$\Gamma[\alpha]$ is satisfiable---that is,
if there is some truth assignment~$\beta : \Var{\Gamma} \backslash \III \rightarrow \BB$
such that~$\alpha \cup \beta$ satisfies~$\Gamma$.

\subsubsection{Preferences over Opinions}

Strategic behavior for judgment aggregation
(such as the problems of manipulation, bribery and control)
involves the incentive to obtain a ``better'' outcome.
Therefore, in order to study strategic behavior,
it is essential to define a notion of preference over opinions---%
that is, when is one opinion ``better than''
(or preferred over) another opinion.

In the worst case, the number of possible opinions that play
a role is exponential in the number of issues---%
e.g., for~$m$ issues there could be up to~$2^m$ possible
opinions.
As a result, it is unreasonable to expect agents to
explicitly specify a preference relation over
all (feasible) opinions.
Instead it makes more sense to use a compact specification
language to represent a preference relation over opinions.
In this paper, we will use one such specification method that
can be used to capture a wide range of preferences.
Various preference relations over opinions have been studied
in the literature \cite{BaumeisterErdelyiErdelyiRothe15,%
DietrichList07,DokowFalik12}.

We consider preferences based on a \emph{weighted Hamming
distance}.
We define this weighted Hamming distance for the setting
of formula-based judgment aggregation.
Definitions for the setting of constraint-based judgment aggregation
are entirely similar.
Take an agenda~$\Phi$ together with an integrity constraint~$\Gamma$.
An agent can specify their preference relation over complete and
$\Gamma$-consistent judgment sets~$J \in \JJJ(\Phi,\Gamma)$
by providing a weight function~$w : \Pre{\Phi} \rightarrow \NN{}$
that produces a weight~$w(\varphi)$ for each formula~$\varphi \in \Pre{\Phi}$.
Intuitively, for each~$\varphi \in \Pre{\Phi}$,
the weight~$w(\varphi)$ indicates how important it is for the agent that
the outcome agrees with their truthful opinion on the issue~$\varphi$.
(Alternatively, one could consider weight functions that produce
rational or real weights.)
Then, for two complete judgment sets~$J_1$ and~$J_2$ the weighted
Hamming distance~$d(J_1,J_2,w)$ is defined as follows:
\[ d(J_1,J_2,w) = \sum \SB w(\varphi) \SM \varphi \in \Pre{\Phi},
\varphi \in (J_1 \setminus J_2) \cup (J_2 \setminus J_1) \SE. \]
That is, for each formula~$\varphi \in \Pre{\Phi}$ that~$J$ and~$J'$
disagree on, the weighted Hamming distance is increased by~$w(\varphi)$.

Using this notion of weighted Hamming distance, we can define a
preference relation for an agent.
Suppose that the agent's truthful opinion is given by a complete
and $\Gamma$-consistent judgment set~$J$.
Moreover, suppose that the agent's view on the relative importance
of the separate issues is given by a weight function~$w : \Pre{\Phi} \rightarrow \NN$.
Then the preference relation~$\leq_{w,J}$ for this agent
is defined as follows.
For any two complete and $\Gamma$-consistent judgment sets~$J_1,J_2$,
it holds:
\[ J_1 \leq_{w,J} J_2 \quad\mtext{if and only if}\quad d(J,J_1,w) \leq d(J,J_2,w). \]
Correspondingly, a judgment set~$J_1$ is (strictly) preferred over another
judgment set~$J_2$ if and only if the weighted Hamming distance
from~$J_1$ to~$J$ is (strictly) smaller than the weighted Hamming distance
from~$J_2$ to~$J$.

A particular case of the weighted Hamming distance is the
\emph{unweighted Hamming distance}.
That is, the case where~$w(\varphi) = 1$ for all~$\varphi \in \Pre{\Phi}$.
Whenever the weight function~$w$ is the constant function that
always returns~$1$, we drop the~``$w$'' from the notation---%
that is, the unweighted Hamming distance between two judgment
sets~$J_1$ and~$J_2$ is denoted by~$d(J_1,J_2)$.

\paragraph{Other preference relations}
In the literature, there have been various proposals for notions
of preference over opinions.
For example, the phenomenon of manipulation in judgment aggregation
has been studied in the settings (1)~where one judgment set is preferred
over another if it agrees with a fixed optimal judgment set on at least
one issue where the other judgment set disagrees \cite{DokowFalik12},
and (2)~where one judgment set is preferred over a second judgment
set if it agrees
with a fixed optimal judgment set on at least one issue where
the second judgment set disagrees, and for all issues it holds that
if the second judgment set agrees with the optimal judgment set
then the first judgment set also agrees with the optimal
\cite{DokowFalik12}.
Other preference relations that have been investigated are
top-respecting preferences and closeness-respecting preferences.
The class of {top-respecting preferences} contains all preferences
that prefer a single most preferred judgment set over all other
judgment sets (and the preference between the other judgment sets
is arbitrary)
\cite{BaumeisterErdelyiErdelyiRothe15,DietrichList07}.
The class of {closeness-respecting preferences} contains preferences
that additionally satisfy the condition of closeness:
if one judgment set agrees with the most preferred judgment
on a superset of issues compared to another judgment set,
then the one judgment is preferred over the other
\cite{BaumeisterErdelyiErdelyiRothe15,DietrichList07}.

\section{Manipulation}
\label{sec:manipulation}

The first form of strategic behavior in judgment aggregation that we
consider is manipulation.
This concerns cases where individuals aim to influence the
outcome of the aggregation
procedure in their favor by reporting an insincere judgment,
that is, by reporting a judgment that differs from their beliefs.

For irresolute judgment aggregation procedures such as the
Kemeny procedure,
one can consider various requirements on the strategically reported
insincere judgments.
For instance, one could require that every outcome for the
insincere judgment is preferred over every outcome for
the sincere judgment.
Alternatively, one could require that there is at least one outcome
for the insincere judgment that is preferred over every outcome for
the sincere judgment.
Correspondingly, we consider the following decision problems.
(We formalize these problems for the setting of formula-based
judgment aggregation.
For the setting of constraint-based judgment aggregation,
these problems are defined entirely similarly.)

\probdef{
  \CautiousManipulation{\Kemeny}
  
  \emph{Instance:} An agenda~$\Phi$ with an
    integrity constraint~$\Gamma$,
    a weight function~$w : \Pre{\Phi} \rightarrow \NN$,
    and a profile~$\prof{J} \in \JJJ(\Phi,\Gamma)^{+}$.
  
  \emph{Question:} Is there a complete and consistent judgment
    set $J' \in \JJJ(\Phi,\Gamma)$
    such that for \textbf{all} $J^{*}_{\mtext{new}} \in \Kemeny(\prof{J}_{-1},J')$
    and for \textbf{all} $J^{*}_{\mtext{old}} \in \Kemeny(\prof{J})$
    it holds that $d(J^{*}_{\mtext{new}},J_1,w) < d(J^{*}_{\mtext{old}},J_1,w)$?
}

\probdef{
  \OptimisticManipulation{\Kemeny}
  
  \emph{Instance:} An agenda~$\Phi$ with an
    integrity constraint~$\Gamma$,
    a weight function~$w : \Pre{\Phi} \rightarrow \NN$,
    and a profile~$\prof{J} \in \JJJ(\Phi,\Gamma)^{+}$.
  
  \emph{Question:} Is there a complete and consistent judgment
    set $J' \in \JJJ(\Phi,\Gamma)$
    and \textbf{some} $J^{*}_{\mtext{new}} \in \Kemeny(\prof{J}_{-1},J')$
    such that for \textbf{all} $J^{*}_{\mtext{old}} \in \Kemeny(\prof{J})$
    it holds that $d(J^{*}_{\mtext{new}},J_1,w) < d(J^{*}_{\mtext{old}},J_1,w)$?
}

\probdef{
  \PessimisticManipulation{\Kemeny}
  
  \emph{Instance:} An agenda~$\Phi$ with an
    integrity constraint~$\Gamma$,
    a weight function~$w : \Pre{\Phi} \rightarrow \NN$,
    and a profile~$\prof{J} \in \JJJ(\Phi,\Gamma)^{+}$.
  
  \emph{Question:} Is there a complete and consistent judgment
    set $J' \in \JJJ(\Phi,\Gamma)$
    such that for \textbf{all} $J^{*}_{\mtext{new}} \in \Kemeny(\prof{J}_{-1},J')$
    there is \textbf{some} $J^{*}_{\mtext{old}} \in \Kemeny(\prof{J})$
    such that $d(J^{*}_{\mtext{new}},J_1,w) < d(J^{*}_{\mtext{old}},J_1,w)$?
}

\probdef{
  \SuperoptimisticManipulation{\Kemeny}
  
  \emph{Instance:} An agenda~$\Phi$ with an
    integrity constraint~$\Gamma$,
    a weight function~$w : \Pre{\Phi} \rightarrow \NN$,
    and a profile~$\prof{J} \in \JJJ(\Phi,\Gamma)^{+}$.
  
  \emph{Question:} Is there a complete and consistent judgment
    set $J' \in \JJJ(\Phi,\Gamma)$,
    \textbf{some} $J^{*}_{\mtext{new}} \in \Kemeny(\prof{J}_{-1},J')$
    and \textbf{some} $J^{*}_{\mtext{old}} \in \Kemeny(\prof{J})$
    such that $d(J^{*}_{\mtext{new}},J_1,w) < d(J^{*}_{\mtext{old}},J_1,w)$?
}

\probdef{
  \SafeManipulation{\Kemeny}
  
  \emph{Instance:} An agenda~$\Phi$ with an
    integrity constraint~$\Gamma$,
    a weight function~$w : \Pre{\Phi} \rightarrow \NN$,
    and a profile~$\prof{J} \in \JJJ(\Phi,\Gamma)^{+}$.
  
  \emph{Question:} Is there a complete and consistent judgment
    set $J' \in \JJJ(\Phi,\Gamma)$
    such that (1)~for \textbf{all} $J^{*}_{\mtext{new}} \in \Kemeny(\prof{J}_{-1},J')$
    and for \textbf{all} $J^{*}_{\mtext{old}} \in \Kemeny(\prof{J})$
    it holds that $d(J^{*}_{\mtext{new}},J_1,w) \leq d(J^{*}_{\mtext{old}},J_1,w)$,
    and such that (2)~there exists
    \textbf{some} $J^{*}_{\mtext{new}} \in \Kemeny(\prof{J}_{-1},J')$
    and \textbf{some} $J^{*}_{\mtext{old}} \in \Kemeny(\prof{J})$
    such that $d(J^{*}_{\mtext{new}},J_1,w) < d(J^{*}_{\mtext{old}},J_1,w)$?
}

\subsection{Complexity Results}

In this section, we prove the following result.

\begin{theorem}
\label{thm:manipulation-theorem}
The following problems are \SigmaP{2}-complete:
\begin{itemize}
  \item \CautiousManipulation{\Kemeny},
  \item \OptimisticManipulation{\Kemeny},
  \item \PessimisticManipulation{\Kemeny},
  \item \SuperoptimisticManipulation{\Kemeny}, and
  \item \SafeManipulation{\Kemeny}.
\end{itemize}
Moreover, \SigmaP{2}-hardness holds already for the case
where the manipulator's preferences are based on the unweighted
Hamming distance.
\end{theorem}

This result follows from Propositions~\ref{prop:manipulation-cautious-membership}--%
\ref{prop:manipulation-superoptimistic-membership}
and~\ref{prop:manipulation-cautious-hardness},
and Corollaries~\ref{cor:manipulation-safe-membership}
and~\ref{cor:manipulation-others-hardness},
that we establish below.

\begin{proposition}
\label{prop:manipulation-cautious-membership}
\CautiousManipulation{\Kemeny} is in \SigmaP{2}.
\end{proposition}
\begin{proof}
We describe a nondeterministic polynomial-time
algorithm with access to an \NP{} oracle
that solves the problem.
Let~$(\Phi,\Gamma,w,\prof{J})$ specify an instance of \CautiousManipulation{\Kemeny}.
The algorithm proceeds in several steps.

Firstly,~(1) the algorithm determines the minimum
distance~$d^{\mtext{win}}_{\mtext{old}}$
from~$\prof{J}$ to any complete and consistent judgment
set~$J^{*} \in \JJJ(\Phi,\Gamma)$.
That is,~$d^{\mtext{win}}_{\mtext{old}}$ is the cumulative
unweighted Hamming distance
from the judgments in~$\prof{J}$ to the judgment sets~$J^{*} \in \Kemeny(\prof{J})$.
This can be done in (deterministic) polynomial time
using~$O(\log n)$ queries to an \NP{} oracle.

Then,~(2) the algorithm determines the minimum
distance~$d^{\mtext{min}}_{\mtext{old}}$ (weighted by~$w$)
from~$J_1$ to any judgment set~$J^{*} \in \Kemeny(\prof{J})$,
that is, from~$J_1$ to any complete and consistent judgment
set~$J^{*}$ that has cumulative unweighted
Hamming distance~$d^{\mtext{win}}_{\mtext{old}}$
to the profile~$\prof{J}$.
This can also be done in (deterministic) polynomial time
using an \NP{} oracle.

Next,~(3a) the algorithm guesses a complete judgment set~$J'$
together with a truth assignment~$\alpha : \Var{\Phi,\Gamma} \rightarrow
\BB{}$, and it checks whether~$\alpha$ satisfies both~$J'$
and~$\Gamma$.
This can be done in nondeterministic polynomial time.
Moreover,~(3b) the algorithm determines the minimum
distance~$d^{\mtext{win}}_{\mtext{new}}$
from~$(\prof{J}_{-1},J')$ to any complete and consistent judgment
set~$J^{*} \in \JJJ(\Phi,\Gamma)$.
Finally,~(3c) the algorithm determines by using a single query to
an \NP{} oracle whether there exists some complete and consistent
judgment set~$J^{*}_{\mtext{new}} \in \JJJ(\Phi,\Gamma)$ such
that~$d(J^{*}_{\mtext{new}},(\prof{J}_{-1},J')) = d^{\mtext{win}}_{\mtext{new}}$
and~$d(J^{*}_{\mtext{new}},J_1,w) \geq d^{\mtext{min}}_{\mtext{old}}$.
If this is the case, the algorithm rejects;
otherwise, the algorithm accepts.

It is straightforward to verify that the algorithm runs in nondeterministic
polynomial time.
Moreover, the algorithm accepts the input (for some sequence of
nondeterministic choices) if and only if there is some complete and
consistent judgment set~$J'$ such that
for all~$J^{*}_{\mtext{new}} \in \Kemeny(\prof{J}_{-1},J')$
and for all~$J^{*}_{\mtext{old}} \in \Kemeny(\prof{J})$
it holds that~$d(J^{*}_{\mtext{new}},J_1,w) < d(J^{*}_{\mtext{old}},J_1,w)$.
\end{proof}

\begin{proposition}
\label{prop:manipulation-optimistic-membership}
\OptimisticManipulation{\Kemeny} is in \SigmaP{2}.
\end{proposition}
\begin{proof}
We describe a nondeterministic polynomial-time
algorithm with access to an \NP{} oracle
that solves the problem.
Let~$(\Phi,\Gamma,w,\prof{J})$ specify an instance of \OptimisticManipulation{\Kemeny}.
The algorithm proceeds in several steps.
For the first two steps, the algorithm proceeds exactly as the
algorithm described in the proof of Proposition~\ref{prop:manipulation-cautious-membership}.
That is,~(1) the algorithm computes~$d^{\mtext{win}}_{\mtext{old}}$
and~(2) it computes~$d^{\mtext{min}}_{\mtext{old}}$,
both in deterministic polynomial time
using an \NP{} oracle.

For the third step, the algorithm proceeds in a similar fashion as the algorithm
described in the proof of Proposition~\ref{prop:manipulation-cautious-membership}.
That is,~(3a) the algorithm guesses a complete judgment set~$J'$
together with a truth assignment~$\alpha : \Var{\Phi,\Gamma} \rightarrow
\BB{}$, and it checks whether~$\alpha$ satisfies both~$J'$
and~$\Gamma$.
Also,~(3b) the algorithm determines the minimum unweighted Hamming
distance~$d^{\mtext{win}}_{\mtext{new}}$
from~$(\prof{J}_{-1},J')$ to any complete and consistent judgment
set~$J^{*} \in \JJJ(\Phi,\Gamma)$.
This can be done using~$O(\log n)$ queries to an \NP{} oracle.
Then,~(3c$'$) the algorithm guesses some complete
judgment set~$J^{*}_{\mtext{new}}$ together with a
truth assignment~$\alpha' : \Var{\Phi,\Gamma} \rightarrow
\BB{}$, and it checks whether~$\alpha'$
satisfies both~$J^{*}_{\mtext{new}}$ and~$\Gamma$.
Moreover, the algorithm checks whether%
~$d(J^{*}_{\mtext{new}},(\prof{J}_{-1},J')) = d^{\mtext{win}}_{\mtext{new}}$---%
that is,~$J^{*}_{\mtext{new}} \in \Kemeny(\prof{J}_{-1},J')$---%
and~$d(J^{*}_{\mtext{new}},J_1,w) < d^{\mtext{min}}_{\mtext{old}}$,
and accepts if and only if this is the case.

It is straightforward to verify that the algorithm runs in nondeterministic
polynomial time.
Moreover, the algorithm accepts the input (for some sequence of
nondeterministic choices) if and only if there is some complete and
consistent judgment set~$J'$ and
some~$J^{*}_{\mtext{new}} \in \Kemeny(\prof{J}_{-1},J')$
such that for all~$J^{*}_{\mtext{old}} \in \Kemeny(\prof{J})$
it holds that~$d(J^{*}_{\mtext{new}},J_1,w) < d(J^{*}_{\mtext{old}},J_1,w)$.
\end{proof}

\begin{proposition}
\label{prop:manipulation-pessimistic-membership}
\PessimisticManipulation{\Kemeny} is in \SigmaP{2}.
\end{proposition}
\begin{proof}
We describe a nondeterministic polynomial-time
algorithm with access to an \NP{} oracle
that solves the problem.
Let~$(\Phi,\Gamma,w,\prof{J})$ specify an instance of \PessimisticManipulation{\Kemeny}.
The algorithm proceeds in several steps.
For the first step, the algorithm proceeds exactly as the
algorithm described in the proof of Proposition~\ref{prop:manipulation-cautious-membership}.
That is,~(1) the algorithm computes~$d^{\mtext{win}}_{\mtext{old}}$
in deterministic polynomial time
using~$O(\log n)$ queries to an \NP{} oracle.

Then,~(2) the algorithm determines the maximum
distance~$d^{\mtext{max}}_{\mtext{old}}$ (weighted by~$w$)
from~$J_1$ to any
judgment set~$J^{*} \in \Kemeny(\prof{J})$,
that is, from~$J_1$ to any complete and consistent judgment
set~$J^{*}$ that has cumulative unweighted Hamming
distance~$d^{\mtext{win}}_{\mtext{old}}$
to the profile~$\prof{J}$.
This can also be done in (deterministic) polynomial time
using an \NP{} oracle.

Next,~(3a) the algorithm guesses a complete judgment set~$J'$
together with a truth assignment~$\alpha : \Var{\Phi,\Gamma} \rightarrow
\BB{}$, and it checks whether~$\alpha$ satisfies both~$J'$
and~$\Gamma$.
This can be done in nondeterministic polynomial time.
Moreover,~(3b) the algorithm determines the minimum unweighted Hamming
distance~$d^{\mtext{win}}_{\mtext{new}}$
from~$(\prof{J}_{-1},J')$ to any complete and consistent judgment
set~$J^{*} \in \JJJ(\Phi,\Gamma)$.
Finally,~(3c) the algorithm determines by using a single query to
an \NP{} oracle whether there exists some complete and consistent
judgment set~$J^{*}_{\mtext{new}} \in \JJJ(\Phi,\Gamma)$ such
that~$d(J^{*}_{\mtext{new}},(\prof{J}_{-1},J')) = d^{\mtext{win}}_{\mtext{new}}$
and~$d(J^{*}_{\mtext{new}},J_1,w) \geq d^{\mtext{max}}_{\mtext{old}}$.
If this is the case, the algorithm rejects;
otherwise, the algorithm accepts.

It is straightforward to verify that the algorithm runs in nondeterministic
polynomial time.
Moreover, the algorithm accepts the input (for some sequence of
nondeterministic choices) if and only if there is some complete and
consistent judgment set~$J'$ such that
for all~$J^{*}_{\mtext{new}} \in \Kemeny(\prof{J}_{-1},J')$
there is some~$J^{*}_{\mtext{old}} \in \Kemeny(\prof{J})$
such that~$d(J^{*}_{\mtext{new}},J_1,w) < d(J^{*}_{\mtext{old}},J_1,w)$.
\end{proof}

\begin{proposition}
\label{prop:manipulation-superoptimistic-membership}
\SuperoptimisticManipulation{\Kemeny} is in \SigmaP{2}.
\end{proposition}
\begin{proof}
We describe a nondeterministic polynomial-time
algorithm with access to an \NP{} oracle
that solves the problem.
Let~$(\Phi,\Gamma,w,\prof{J})$ specify an instance of
\SuperoptimisticManipulation{\Kemeny}.
The algorithm proceeds in several steps.
For the first steps, the algorithm proceeds exactly as the
algorithm described in the proof of Proposition~\ref{prop:manipulation-cautious-membership}.
That is,~(1) the algorithm computes~$d^{\mtext{win}}_{\mtext{old}}$,
in deterministic polynomial time
using~$O(\log n)$ queries to an \NP{} oracle.

Then,~(2$'$) the algorithm guesses a complete judgment set~$J^{*}_{\mtext{old}}$
together with a truth assignment~$\gamma : \Var{\Phi,\Gamma} \rightarrow
\BB{}$, and it checks whether~$\gamma$ satisfies both~$J^{*}_{\mtext{old}}$
and~$\Gamma$, and whether~$d(J^{*}_{\mtext{old}},\prof{J}) = d^{\mtext{win}}_{\mtext{old}}$.

For the third step, the algorithm proceeds in a similar fashion as the algorithm
described in the proof of Proposition~\ref{prop:manipulation-cautious-membership}.
That is,~(3a) the algorithm guesses a complete judgment set~$J'$
together with a truth assignment~$\alpha : \Var{\Phi,\Gamma} \rightarrow
\BB{}$, and it checks whether~$\alpha$ satisfies both~$J'$
and~$\Gamma$.
Also,~(3b) the algorithm determines the minimum
distance~$d^{\mtext{win}}_{\mtext{new}}$
from~$(\prof{J}_{-1},J')$ to any complete and consistent judgment
set~$J^{*} \in \JJJ(\Phi,\Gamma)$.
This can be done using~$O(\log n)$ queries to an \NP{} oracle.
Then,~(3c$''$) the algorithm guesses some complete
judgment set~$J^{*}_{\mtext{new}}$ together with a
truth assignment~$\alpha' : \Var{\Phi,\Gamma} \rightarrow
\BB{}$, and it checks whether~$\alpha'$
satisfies both~$J^{*}_{\mtext{new}}$ and~$\Gamma$.
Moreover, the algorithm checks whether%
~$d(J^{*}_{\mtext{new}},(\prof{J}_{-1},J')) = d^{\mtext{win}}_{\mtext{new}}$---%
that is,~$J^{*}_{\mtext{new}} \in \Kemeny(\prof{J}_{-1},J')$---%
and~$d(J^{*}_{\mtext{new}},J_1,w) < d(J^{*}_{\mtext{old}},J_1,w)$,
and accepts if and only if this is the case.

It is straightforward to verify that the algorithm runs in nondeterministic
polynomial time.
Moreover, the algorithm accepts the input (for some sequence of
nondeterministic choices) if and only if there is some complete and
consistent judgment set~$J'$,
some~$J^{*}_{\mtext{new}} \in \Kemeny(\prof{J}_{-1},J')$
and some~$J^{*}_{\mtext{old}} \in \Kemeny(\prof{J})$
such that~$d(J^{*}_{\mtext{new}},J_1,w) < d(J^{*}_{\mtext{old}},J_1,w)$.
\end{proof}

\begin{corollary}
\label{cor:manipulation-safe-membership}
\SafeManipulation{\Kemeny} is in \SigmaP{2}.
\end{corollary}
\begin{proof}[Proof (sketch)]
The algorithms described in the proofs of
Proposition~\ref{prop:manipulation-cautious-membership}
and~\ref{prop:manipulation-superoptimistic-membership}
can straightforwardly be modified and combined to form
a nondeterministic polynomial-time algorithm with access to an \NP{} oracle
that solves \SafeManipulation{\Kemeny}.
\end{proof}

\begin{proposition}
\label{prop:manipulation-cautious-hardness}
\CautiousManipulation{\Kemeny} is \SigmaP{2}-hard.
\end{proposition}
\begin{proof}
We show \SigmaP{2}-hardness by giving a reduction from the satisfiability
problem for quantified Boolean formulas of the form~$\exists x_1,\dotsc,x_n.
\forall y_1,\dotsc,y_m. \psi$.
Let~$\varphi = \exists x_1,\dotsc,x_n. \forall y_1,\dotsc,y_m. \psi$ be a
quantified Boolean formula.
Let~$X = \SBs x_1,\dotsc,x_n \SEs$ and~$Y = \SBs y_1,\dotsc,y_m \SEs$.
Without loss of generality, we assume that~$n$ is a multiple of~$3$---%
that is, that~$n = 3n'$ for some~$n' \in \NN$.
%
We construct an agenda~$\Phi$, an integrity constraint~$\Gamma$,
a weight function~$w : \Pre{\Phi} \rightarrow \NN$,
and a profile~$\prof{J}$ as follows.

We consider fresh variables~$x'_1,\dotsc,x'_n$
and fresh variables~$y'_1,\dotsc,y'_m$.
Moreover, we consider fresh variables~$z_1,\dotsc,z_{n/3}$,
fresh variables~$t_1,\dotsc,t_m$,
fresh variables~$w_{1,\ell}$
for~$\ell \in [n+1]$,
fresh variables~$w_{2,\ell}$
for~$i \in \SBs 2,3 \SEs$ and~$\ell \in [n]$,
and fresh variables~$u_{i,\ell}$ for~$i \in [3]$ and~$\ell \in [u]$,
where~$u = 10n+10m+10$.
We then let the agenda~$\Phi$ consist of the variables~$x_1,\dotsc,x_n$
and~$y_1,\dotsc,y_m$ and the fresh variables we introduced above,
together with their negations.
That is, we let~$\Pre{\Phi} = \SB x_i, x'_i \SM i \in [n] \SE \cup
\SB z_i \SM i \in [n/3] \SE \cup
\SB y_j, y'_j, t_j \SM j \in [m] \SE \cup
\SB w_{1,\ell} \SM \ell \in [n+1] \SE \cup
\SB w_{i,\ell} \SM i \in \SBs 2,3 \SEs, \ell \in [n] \SE \cup
\SB u_{i,\ell} \SM i \in [3], \ell \in [u] \SE$.


We then define the integrity constraint~$\Gamma$ as follows.
We let
\[ \Gamma = \Gamma_0 \vee
\bigvee\limits_{i \in [3]} \bigwedge\limits_{\ell \in \mathrlap{[u]}} u_{i,\ell},\]
and
\begin{align*}
  \Gamma_0 =\ &
    \left (
      \left ( \bigwedge\limits_{i \in [n]} (x_i \oplus x'_i) \right )
    \vee
      \bigwedge\limits_{i \in [n/\mathrlap{3]}} z_i
    \right )
      \wedge
    \left (
      \left ( \bigwedge\limits_{i \in [n]} (x_i \oplus x'_i) \right )
    \rightarrow
      \bigwedge\limits_{i \in [n/\mathrlap{3]}} \neg z_i
    \right ) \\
  \wedge\ &
    \left (
      \left ( \bigwedge\limits_{j \in [m]} (y_j \oplus y'_j) \right )
    \oplus
      \bigwedge\limits_{j \in [m]} t_j
    \right )
      \wedge
    \left (
      \left ( \bigwedge\limits_{j \in [m]} (y_j \oplus y'_j) \right )
    \rightarrow
      \bigwedge\limits_{j \in [m]} \neg t_j
    \right ) \\
  \wedge\ & \left (
      \bigwedge\limits_{i \in [n/\mathrlap{3]}} z_i
    \rightarrow
      \left ( \bigwedge\limits_{i \in [n]} (\neg x_i \wedge \neg x'_i) \right )
    \right )
      \wedge
    \left (
      \left ( \bigwedge\limits_{j \in [m]} t_j \right )
    \rightarrow
      \bigwedge\limits_{j \in [m]} (\neg y_j \wedge \neg y'_j)
    \right ) \\
  \wedge\ &
    \left (
      \left ( \bigwedge\limits_{i \in [n/\mathrlap{3]}} z_i \right )
    \rightarrow
      \bigwedge\limits_{j \in [m]} t_j
    \right ) \\
  \wedge\ &
    \left (
      \left ( \bigwedge\limits_{i \in [n/\mathrlap{3]}} z_i
      \wedge \bigwedge\limits_{j \in [m]} t_j \right )
    \rightarrow
      \left (
      \bigwedge\limits_{\ell \in [n]} w_{2,\ell}
      \vee
      \bigwedge\limits_{\ell \in [n]} w_{3,\ell}
      \right )
    \right ) \\
  \wedge\ &
    \left (
      \left ( \neg \left (\bigwedge\limits_{i \in [n/\mathrlap{3]}} z_i \right )
      \wedge \bigwedge\limits_{j \in [m]} t_j \right )
    \rightarrow
      \bigwedge\limits_{\ell \in [n+1]} w_{1,\ell}
    \right ) \\
  \wedge\ &
    \left (
      \left ( \neg \left ( \bigwedge\limits_{i \in [n/\mathrlap{3]}} z_i \right )
      \wedge \neg \left ( \bigwedge\limits_{j \in [m]} t_j \right ) \right )
    \rightarrow
      \neg\psi \wedge 
      \left (
      \bigwedge\limits_{\ell \in [n]} w_{2,\ell}
      \vee
      \bigwedge\limits_{\ell \in [n]} w_{3,\ell}
      \right )
    \right ).
\end{align*}
(Here~$\oplus$ denotes exclusive disjunction.)

As a result of the definition of~$\Gamma$,
each complete and consistent judgment set~$J \in \JJJ(\Phi,\Gamma)$
satisfies (at least) one of the following four conditions.
\begin{enumerate}
  \item For some~$i \in [3]$, the judgment set~$J$ includes
    each formula~$u_{i,\ell}$ for~$\ell \in [u]$.
  \item The judgment set~$J$ includes
    exactly one of~$x_i$ and~$x'_i$ for each~$i \in [n]$,
    it includes none of the formulas~$z_i$,
    it includes exactly one of~$y_j$ and~$y'_j$ for each~$j \in [m]$,
    it includes none of the formulas~$t_j$,
    it includes either all formulas~$w_{2,\ell}$ or all formulas~$w_{3,\ell}$
    for~$\ell \in [n]$,
    and~$J$ does not satisfy~$\psi$.
  \item The judgment set~$J$ includes
    exactly one of~$x_i$ and~$x'_i$ for each~$i \in [n]$,
    it includes none of the formulas~$z_i$,
    it includes all formulas~$t_j$,
    for each~$j \in [m]$ it includes either none or both of~$y_j$ and~$y'_j$,
    and it includes all formulas~$w_{1,\ell}$ for~$\ell \in [n+1]$.
  \item The judgment set~$J$ includes
    all formulas~$z_i$,
    for each~$i \in [n]$ it includes either none or both of~$x_i$ and~$x'_i$,
    it includes all formulas~$t_j$,
    for each~$j \in [m]$ it includes either none or both of~$y_j$ and~$y'_j$,
    and
    it includes either all formulas~$w_{2,\ell}$ or all formulas~$w_{3,\ell}$
    for~$\ell \in [n]$.
\end{enumerate}

We define~$w : \Pre{\Phi} \rightarrow \NN$
by letting~$w(\varphi) = 1$ for each~$\varphi \in \Pre{\Phi}$.
In other words, we consider the unweighted Hamming distance.

Finally, we let~$\prof{J} = (J_1,J_2,J_3)$, where~$J_1,J_2,J_3$ are
defined as described in Table~\ref{table:manipulation-cautious-hardness-profile}.
In this table, the indices~$i,j,\ell$ range over all possible
values, and for each~$\varphi \in \Pre{\Phi}$ we write a~$0$
if~$\varphi \not\in J_i$ and a~$1$ if~$\varphi \in J_i$.

\begin{table}[ht!]
  \begin{center}
  \begin{tabular}{c | c @{\ \ } c @{\ \ } c @{\ \ } c @{\ \ } c @{\ \ } c @{\ \ } c @{\ \ } c @{\ \ } c @{\ \ } c @{\ \ } c @{\ \ } c}
    \toprule
    $\prof{J}$ & $x_i$ & $x'_i$ & $z_i$ & $y_j$ & $y'_j$ & $t_j$ &
       $w_{1,\ell}$ & $w_{2,\ell}$ & $w_{3,\ell}$ &
       $u_{1,\ell}$ & $u_{2,\ell}$ & $u_{3,\ell}$ \\
    \midrule
    $J_1$ & $0$ & $0$ & $0$ & $0$ & $0$ & $0$ &
      $1$ & $0$ & $0$ & $1$ & $0$ & $0$ \\
    $J_2$ & $0$ & $0$ & $0$ & $0$ & $0$ & $0$ &
      $0$ & $1$ & $0$ & $0$ & $1$ & $0$ \\
    $J_3$ & $1$ & $1$ & $0$ & $0$ & $0$ & $0$ &
      $0$ & $0$ & $1$ & $0$ & $0$ & $1$ \\
    \bottomrule
  \end{tabular}
  \end{center}
  \caption{The profile~$\prof{J} = (J_1,J_2,J_3)$ that we use in the
    proof of Proposition~\ref{prop:manipulation-cautious-hardness}.}
  \label{table:manipulation-cautious-hardness-profile}
\end{table}

In the remainder,
we will argue that there is some truth assignment~$\alpha : X \rightarrow \BB{}$
such that for all truth assignments~$\beta : Y \rightarrow \BB{}$
it holds that~$\psi[\alpha \cup \beta]$ is true
if and only if
there is some complete and consistent
judgment set~$J'_1 \in \JJJ(\Phi,\Gamma)$
such that for all~$J^{*}_{\mtext{new}} \in \Kemeny(\prof{J}_{-1},J')$
and for all~$J^{*}_{\mtext{old}} \in \Kemeny(\prof{J})$
it holds that $d(J^{*}_{\mtext{new}},J_1) < d(J^{*}_{\mtext{old}},J_1)$.

Firstly, we observe that~$\Kemeny(\prof{J}) = \SBs J^{*}_{\mtext{old},1},
J^{*}_{\mtext{old},2} \SEs$, where~$J^{*}_{\mtext{old},1}$
and~$J^{*}_{\mtext{old},2}$ are defined as described
in Table~\ref{table:manipulation-cautious-hardness-old-winners}.
Both~$J^{*}_{\mtext{old},1}$ and~$J^{*}_{\mtext{old},2}$ are complete
and consistent, and have a cumulative Hamming distance of~$7n+3m+1+3u$
to the profile~$\prof{J}$.
It is straightforward to verify that no complete and consistent judgment set
has a smaller cumulative Hamming distance to the profile~$\prof{J}$.
(For instance, the judgment sets of the form~$J^{*}_{\alpha}$,
as described below in
Table~\ref{table:manipulation-cautious-hardness-insincere-judgment},
have a cumulative Hamming distance of~$7n+3m+2+3u$ to the profile~$\prof{J}$.)

\begin{table}[ht!]
  \begin{center}
  \begin{tabular}{c | c @{\ \ } c @{\ \ } c @{\ \ } c @{\ \ } c @{\ \ } c @{\ \ } c @{\ \ } c @{\ \ } c @{\ \ } c @{\ \ } c @{\ \ } c}
    \toprule
    $\Kemeny(\prof{J})$ & $x_i$ & $x'_i$ & $z_i$ & $y_j$ & $y'_j$ & $t_j$ &
       $w_{1,\ell}$ & $w_{2,\ell}$ & $w_{3,\ell}$ &
       $u_{1,\ell}$ & $u_{2,\ell}$ & $u_{3,\ell}$ \\
    \midrule
    $J^{*}_{\mtext{old},1}$ & $0$ & $0$ & $1$ & $0$ & $0$ & $1$ &
      $0$ & $1$ & $0$ & $0$ & $0$ & $0$ \\
    $J^{*}_{\mtext{old},2}$ & $0$ & $0$ & $1$ & $0$ & $0$ & $1$ &
      $0$ & $0$ & $1$ & $0$ & $0$ & $0$ \\
    \bottomrule
  \end{tabular}
  \end{center}
  \caption{The judgment sets~$J^{*}_{\mtext{old},1}$
    and~$J^{*}_{\mtext{old},2}$ that are used in the
    proof of Proposition~\ref{prop:manipulation-cautious-hardness}.}
  \label{table:manipulation-cautious-hardness-old-winners}
\end{table}

Next, we argue that the only way for agent~$1$ to enforce an
outcome that is closer to~$J_1$ than~$J^{*}_{\mtext{old},1}$
and~$J^{*}_{\mtext{old},2}$ are, is to enforce an outcome~$J^{*}$
that satisfies condition~(3).
It is impossible for agent~$1$ to force a complete and consistent
judgment set that satisfies condition~(1) to be a Kemeny outcome,
because there are complete and consistent judgment sets at
smaller distance to the profile for any judgment~$J'_1$ that agent~$1$ reports
(e.g., the judgment sets~$J^{*}_{\mtext{old},1}$
and~$J^{*}_{\mtext{old},2}$).
For each of the complete and consistent judgment sets that satisfy
condition~(2), the Hamming distance to~$J_1$ is at least
as much as the Hamming distance from~$J_1$
to~$J^{*}_{\mtext{old},1}$ and~$J^{*}_{\mtext{old},2}$.
Among the complete and consistent judgment sets that satisfy
condition~(4), the judgment sets~$J^{*}_{\mtext{old},1}$
and~$J^{*}_{\mtext{old},2}$ are of smallest Hamming distance
to~$J_1$. Therefore, there is no insincere judgment~$J'_1$
that agent~$1$ can report to obtain a Kemeny outcome
satisfying condition~(4) that is of smaller Hamming distance
to~$J_1$.
This means that if agent~$1$ were to enforce
a Kemeny outcome~$J^{*}_{\mtext{new}}$ that is of smaller
Hamming distance to~$J_1$
than~$J^{*}_{\mtext{old},1}$ and~$J^{*}_{\mtext{old},2}$ are
(by reporting an insincere judgment~$J'_1$),
then~$J^{*}_{\mtext{new}}$ must satisfy condition~(3).

We continue with observing several further properties
that each complete and consistent judgment set~$J^{*}_{\mtext{new}}
\in \Kemeny(\prof{J}_{-1},J')$ must satisfy.
We know that~$J^{*}_{\mtext{new}}$ includes none of the
formulas~$y_j$ and~$y'_j$, as including these would strictly
increase the Hamming distance to the profile.
Similarly,~$J^{*}_{\mtext{new}}$ includes none of the
formulas~$w_{2,\ell}$,~$w_{3,\ell}$ and~$u_{i,\ell}$.

Finally,
we argue that there is some truth assignment~$\alpha : X \rightarrow \BB{}$
such that for all truth assignments~$\beta : Y \rightarrow \BB{}$
it holds that~$\psi[\alpha \cup \beta]$ is true
if and only if
there is some complete and consistent
judgment set~$J'_1 \in \JJJ(\Phi,\Gamma)$
such that for all~$J^{*}_{\mtext{new}} \in \Kemeny(\prof{J}_{-1},J')$
and for all~$J^{*}_{\mtext{old}} \in \Kemeny(\prof{J})$
it holds that $d(J^{*}_{\mtext{new}},J_1) < d(J^{*}_{\mtext{old}},J_1)$.

$(\Rightarrow)$
Suppose that there exists some truth assignment~$\alpha : X \rightarrow \BB{}$
such that for all truth assignments~$\beta : Y \rightarrow \BB{}$
it holds that~$\psi[\alpha \cup \beta]$ is true.
Consider the complete and consistent judgment set~$J_{\alpha}$
that is described in Table~\ref{table:manipulation-cautious-hardness-insincere-judgment}.
We then get that~$\Kemeny(\prof{J}_{-1},J_{\alpha}) = \SBs J^{*}_{\alpha} \SEs$,
where~$J^{*}_{\alpha}$ is
also described in Table~\ref{table:manipulation-cautious-hardness-insincere-judgment}.
The only possible complete and consistent judgment sets~$J^{*}$
that could have a smaller Hamming distance to the
profile~$(\prof{J}_{-1},J_{\alpha})$ would have to satisfy condition~(2).
That is, such judgment sets~$J^{*}$ would have to include exactly
one of~$y_j$ and~$y'_j$, for each~$j \in [m]$,
and would have to satisfy~$\neg\psi$.
Moreover, in order for such judgment sets~$J^{*}$ to have a smaller
Hamming distance to the profile, it would have to agree with~$J^{*}_{\alpha}$
on the formulas~$x_i$ and~$x'_i$, for all~$i \in [n]$.
However, since~$\psi[\alpha]$ is valid, we know that such judgment
sets~$J^{*}$ are not consistent.
Therefore,~$\Kemeny(\prof{J}_{-1},J_{\alpha}) = \SBs J^{*}_{\alpha} \SEs$.
Clearly,~$d(J_1,J^{*}_{\alpha}) < d(J_1,J^{*}_{\mtext{old},i})$
for both~$i \in [2]$.
In other words,
for all~$J^{*}_{\mtext{new}} \in \Kemeny(\prof{J}_{-1},J_{\alpha})$
and for all~$J^{*}_{\mtext{old}} \in \Kemeny(\prof{J})$
it holds that $d(J^{*}_{\mtext{new}},J_1) < d(J^{*}_{\mtext{old}},J_1)$.

\begin{table}[ht!]
  \begin{center}
  \begin{tabular}{c | c @{\ \ } c @{\ \ } c @{\ \ } c @{\ \ } c @{\ \ } c @{\ \ } c @{\ \ } c @{\ \ } c @{\ \ } c @{\ \ } c @{\ \ } c}
    \toprule
     & $x_i$ & $x'_i$ & $z_i$ & $y_j$ & $y'_j$ & $t_j$ &
       $w_{1,\ell}$ & $w_{2,\ell}$ & $w_{3,\ell}$ &
       $u_{1,\ell}$ & $u_{2,\ell}$ & $u_{3,\ell}$ \\
    \midrule
    $J_{\alpha}$ & $\alpha(x_i)$ & $1-\alpha(x_i)$ & $0$ & $0$ & $0$ & $1$ &
      $1$ & $0$ & $0$ & $1$ & $0$ & $0$ \\
    $J^{*}_{\alpha}$ & $\alpha(x_i)$ & $1-\alpha(x_i)$ & $0$ & $0$ & $0$ & $1$ &
      $1$ & $0$ & $0$ & $0$ & $0$ & $0$ \\
    \bottomrule
  \end{tabular}
  \end{center}
  \caption{The judgment sets~$J_{\alpha}$
    and~$J^{*}_{\alpha}$ that are used in
    proof of Proposition~\ref{prop:manipulation-cautious-hardness}.}
  \label{table:manipulation-cautious-hardness-insincere-judgment}
\end{table}

$(\Leftarrow)$
Conversely, suppose that there is some complete and consistent
judgment set~$J'_1 \in \JJJ(\Phi,\Gamma)$
such that for all~$J^{*}_{\mtext{new}} \in \Kemeny(\prof{J}_{-1},J'_1)$
and for all~$J^{*}_{\mtext{old}} \in \Kemeny(\prof{J})$
it holds that $d(J^{*}_{\mtext{new}},J_1) < d(J^{*}_{\mtext{old}},J_1)$.
As we argued above, each such complete and consistent
judgment set~$J^{*}_{\mtext{new}}$ satisfies condition~(3),
and furthermore, it does not contain any of the
formulas~$y_j,y'_j,w_{2,\ell},w_{3,\ell}$ and~$u_{i,\ell}$.
That is, such a judgment set~$J^{*}_{\mtext{new}}$
must be of the form~$J^{*}_{\alpha}$,
for some~$\alpha : X \rightarrow \BB{}$,
as described in
Figure~\ref{table:manipulation-cautious-hardness-insincere-judgment}.
Fix this truth assignment~$\alpha : X \rightarrow \BB{}$.
It suffices to consider insincere judgment sets~$J'_1$ that minimize
the distance to the judgment set~$J^{*}_{\alpha}$---that is,
it suffices to consider the situation where~$J'_1 = J^{*}_{\alpha}$.
(The following argument works regardless of the judgment of~$J'_1$
on the formulas~$u_{1,\ell}$.)
In other words, we consider the profile~$\prof{J}' = (\prof{J}_{-1},J^{*}_{\alpha})$.

\begin{table}[ht!]
  \begin{center}
  \begin{tabular}{c | c @{\ \ } c @{\ \ } c @{\ \ } c @{\ \ } c @{\ \ } c @{\ \ } c @{\ \ } c @{\ \ } c @{\ \ } c @{\ \ } c @{\ \ } c}
    \toprule
     & $x_i$ & $x'_i$ & $z_i$ & $y_j$ & $y'_j$ & $t_j$ &
       $w_{1,\ell}$ & $w_{2,\ell}$ & $w_{3,\ell}$ &
       $u_{1,\ell}$ & $u_{2,\ell}$ & $u_{3,\ell}$ \\
    \midrule
    $J^{*}_{\beta}$ & $\alpha(x_i)$ & $1-\alpha(x_i)$ & $0$ & $\beta(y_j)$ & $1-\beta(y_j)$ & $0$ &
      $0$ & $1$ & $0$ & $0$ & $0$ & $0$ \\
    \bottomrule
  \end{tabular}
  \end{center}
  \caption{The judgment set~$J^{*}_{\beta}$ that is used in
    the proof of Proposition~\ref{prop:manipulation-cautious-hardness}.}
  \label{table:manipulation-cautious-hardness-counterexample}
\end{table}

We now use the fact that~$J^{*}_{\alpha} \in \Kemeny(\prof{J}')$ to argue
that for all truth assignments~$\beta : Y \rightarrow \BB{}$
it holds that~$\psi[\alpha \cup \beta]$ is true.
We proceed indirectly, and assume that there exists some truth
assignment~$\beta : Y \rightarrow \BB{}$
such that~$\psi[\alpha \cup \beta]$ is false.
Then consider the complete and consistent judgment set~$J^{*}_{\beta}$
that is described in Table~\ref{table:manipulation-cautious-hardness-counterexample}.
Because~$\psi[\alpha \cup \beta]$ is false, we know that~$J^{*}_{\beta}$
is consistent---it satisfies condition~(2).
Moreover,~$d(J^{*}_{\beta},\prof{J}') < d(J^{*}_{\alpha},\prof{J}')$.
Namely, the judgment set~$J^{*}_{\alpha}$ has Hamming distance~$6n+3m+2+2u$
to the profile~$\prof{J}'$,
and the judgment set~$J^{*}_{\beta}$ has Hamming distance~$6n+3m+1+2u$
to the profile~$\prof{J}'$.
This is a contradiction with the fact that~$J^{*}_{\alpha} \in \Kemeny(\prof{J}')$.
Therefore, we can conclude that no truth
assignment~$\beta : Y \rightarrow \BB{}$ exists
such that~$\psi[\alpha \cup \beta]$ is false.
In other words, for all truth assignments~$\beta : Y \rightarrow \BB{}$
it holds that~$\psi[\alpha \cup \beta]$ is true.
\end{proof}

\begin{corollary}
\label{cor:manipulation-others-hardness}
\OptimisticManipulation{\Kemeny},
\PessimisticManipulation{\Kemeny},
\SuperoptimisticManipulation{\Kemeny},
and \SafeManipulation{\Kemeny}
are \SigmaP{2}-hard.
\end{corollary}
\begin{proof}
\SigmaP{2}-hardness for \OptimisticManipulation{\Kemeny},
\PessimisticManipulation{\Kemeny},
\SuperoptimisticManipulation{\Kemeny}, and \SafeManipulation{\Kemeny}
follows from the proof of Proposition~\ref{prop:manipulation-cautious-hardness}.
In this proof, all Kemeny outcomes for the original profile~$\prof{J}$
have the same Hamming distance to the judgment set~$J_1$.
Similarly, all Kemeny outcomes for each
profile~$\prof{J}' = (\prof{J}_{-1},J'_1)$,
for each relevant insincere judgment set~$J'_1$,
have the same distance to the judgment set~$J_1$.
Therefore, the reduction can also be used to show \SigmaP{2}-hardness
for the problems \OptimisticManipulation{\Kemeny},
\PessimisticManipulation{\Kemeny},
\SuperoptimisticManipulation{\Kemeny}, and \SafeManipulation{\Kemeny}.
\end{proof}

\subsection{3-Clauses and Trivial Integrity Constraints}
\label{sec:manipulation-3clauses-trivial-ic}

The result of Theorem~\ref{thm:manipulation-theorem} can straightforwardly be
extended to the setting where~$\Gamma = \top$ and where all
formulas~$\varphi \in \Pre{\Phi}$ are clauses of size~$3$.
Clearly, membership in \SigmaP{2} holds also for this restricted setting.
In order to show \SigmaP{2}-hardness,
the proof of Proposition~\ref{prop:manipulation-cautious-hardness}
can be modified as follows, in two steps.
First, we take many syntactic variants~$\chi_1,\dotsc,\chi_{u}$
of the integrity constraint~$\Gamma$, and add these (and their negations)
to the agenda.
Moreover, all judgment sets in the profile~$\prof{J}$ include these
formulas~$\chi_1,\dotsc,\chi_{u}$.
Secondly, we transform each formula~$\chi_i$ into a 3CNF
formula~$\chi'_i$
(by using the standard Tseitin transformation), and replace~$\chi_i$
in the agenda by the clauses of~$\chi'_i$ (and their negations).
Then, all judgment sets in the profile~$\prof{J}$ include all clauses
of all formulas~$\chi'_i$.

\subsection{Constraint-Based Judgment Aggregation}

The result of Theorem~\ref{thm:manipulation-theorem} can also straightforwardly be
extended to the setting of constraint-based judgment aggregation.
The algorithms described in Propositions~\ref{prop:manipulation-cautious-membership}--%
\ref{prop:manipulation-superoptimistic-membership} and
Corollary~\ref{cor:manipulation-safe-membership}
can directly be used in the
setting of constraint-based judgment aggregation as well
(so membership in \SigmaP{2} carries over to this setting).
Since the proofs of Proposition~\ref{prop:manipulation-cautious-hardness}
and Corollary~\ref{cor:manipulation-others-hardness} use an
agenda containing only propositional variables (and their negations),
this proof can also directly be used to show \SigmaP{2}-hardness for the
setting of constraint-based judgment aggregation.

\subsection{Exact Manipulation}

In the literature, the problem of manipulation in judgment aggregation has also been
modelled using different decision problems (see,
e.g.,~\cite{BaumeisterErdelyiErdelyiRothe15}).
One of these decision problems asks whether the manipulator can report
an insincere opinion to obtain an outcome that
includes a given desired subset of the agenda.
This problem is often called \emph{exact manipulation}.
We consider two variants of this exact manipulation problem,
and we argue that the \SigmaP{2}-completeness result of
Theorem~\ref{thm:manipulation-theorem} extends to these problems.

\probdef{
  \CautiousExactManipulation{\Kemeny}
  
  \emph{Instance:} An agenda~$\Phi$ with an
    integrity constraint~$\Gamma$,
    a (possibly incomplete) judgment set~$L \subseteq \Phi$,
    and a profile~$\prof{J} \in \JJJ(\Phi,\Gamma)^{+}$.
  
  \emph{Question:} Is there a complete and consistent judgment
    set $J' \in \JJJ(\Phi,\Gamma)$
    such that for \textbf{all} $J^{*} \in \Kemeny(\prof{J}_{-1},J')$
    it holds that $L \subseteq J^{*}$?
}

\probdef{
  \BraveExactManipulation{\Kemeny}
  
  \emph{Instance:} An agenda~$\Phi$ with an
    integrity constraint~$\Gamma$,
    a (possibly incomplete) judgment set~$L \subseteq \Phi$,
    and a profile~$\prof{J} \in \JJJ(\Phi,\Gamma)^{+}$.
  
  \emph{Question:} Is there a complete and consistent judgment
    set $J' \in \JJJ(\Phi,\Gamma)$
    such that for \textbf{some} $J^{*} \in \Kemeny(\prof{J}_{-1},J')$
    it holds that $L \subseteq J^{*}$?
}

\begin{proposition}
\label{prop:exact-bribery}
The problems \CautiousExactManipulation{\Kemeny} and
\BraveExactManipulation{\Kemeny} are \SigmaP{2}-complete.
Moreover, \SigmaP{2}-hardness holds even for the case
where~$\Card{L} = 1$.
\end{proposition}
\begin{proof}
To show membership in \SigmaP{2}, it suffices to observe that
the algorithms used in the proofs of
Propositions~\ref{prop:manipulation-cautious-membership}
and~\ref{prop:manipulation-optimistic-membership}
can straightforwardly
be extended to the case of \CautiousExactManipulation{\Kemeny} and
\BraveExactManipulation{\Kemeny}.
For \SigmaP{2}-hardness, one can use the proofs of
Proposition~\ref{prop:manipulation-cautious-hardness}
and Corollary~\ref{cor:manipulation-others-hardness},
with the modification that~$L = \SBs w_{1,1} \SEs$.
\end{proof}

\subsection{Group Manipulation}

The result of Theorem~\ref{thm:manipulation-theorem}
also holds for the setting where a coalition of individuals works together
by each reporting an insincere judgment with the aim of obtaining a group
judgment that is preferable for each of the individuals in the coalition
\cite{BotanNovaroEndriss16}.
This group manipulation scenario is a generalization of the
(individual) manipulation scenario that we study in this
paper, since individual manipulation coincides with group manipulation
for a coalition of size~1.
The algorithms used to show membership in \SigmaP{2}
can be straightforwardly modified to work also for group manipulation
for the Kemeny procedure.
Moreover, the \SigmaP{2}-hardness results directly carry over
to the setting of group manipulation for the Kemeny procedure,
since individual manipulation is a special case of group manipulation.

A more subtle form of group manipulation is that where each individual
in the coalition is required to have no incentive to unilaterally leave the
coalition (and report their truthful judgment)
\cite{BotanNovaroEndriss16}.
The \SigmaP{2}-hardness results from this paper carry over to this setting
for the Kemeny procedure,
since the scenario of individual manipulation is also a special case
of the manipulation problem for fragile coalitions.
Additionally, the \SigmaP{2}-membership proofs for the individual manipulation
problem can be modified to work also for the group manipulation problem
for fragile coalitions for the Kemeny procedure.

\section{Bribery}
\label{sec:bribery}

Another form of strategic behavior in judgment aggregation is
bribery.
In this setting, an external agent wishes to influence the outcome of
a judgment aggregation scenario by bribing a number of individuals.
We model this scenario as follows.
(We consider the formula-based judgment aggregation framework;
for the constraint-based judgment aggregation framework, this
bribing scenario can be defined entirely similarly.)

A number of individuals are performing judgment aggregation
on an agenda~$\Phi$ in the presence of an integrity constraint~$\Gamma$.
The briber has a desired (complete and $\Gamma$-consistent)
judgment set~$J_0$, and a weight function~$w : \Pre{\Phi} \rightarrow \NN$
that indicates the relative importance of the different issues for the briber.
Additionally, the briber has a budget that suffices to bribe at most~$k$
individuals.
For all bribed individuals, the briber can specify an arbitrary (complete and
$\Gamma$-consistent) judgment set.
The question is to determine whether the briber can pick up to~$k$
individuals and specify judgment sets for these individuals so that the
outcome of the judgment aggregation procedure is better
(with respect to~$J_0$ and~$w$) than without bribing.

To argue that such bribery can not be done
easily in all possible situations, one can establish computational
intractability results, that give evidence that there are no efficient
algorithms that an external agent can use (across the board)
to obtain a strategy for bribery.

Similarly to the case for manipulation,
we can consider various requirements on the outcomes after
bribery (in relation to the outcomes before bribery).
For instance, we could require that every outcome after
bribery is preferred over every outcome before bribery.
Correspondingly, we consider the following decision problems.

\probdef{
  \CautiousBribery{\Kemeny}
  
  \emph{Instance:} An agenda~$\Phi$ with an
    integrity constraint~$\Gamma$,
    a weight function~$w : \Pre{\Phi} \rightarrow \NN$,
    a profile~$\prof{J} \in \JJJ(\Phi,\Gamma)^{+}$,
    a judgment set~$J_0 \in \JJJ(\Phi,\Gamma)$,
    and an integer~$k \in \NN$.
  
  \emph{Question:} Is it possible to change up to~$k$ individual
    judgment sets in~$\prof{J}$, resulting in a new profile~$\prof{J}'$,
    so that for \textbf{all} $J^{*}_{\mtext{new}} \in \Kemeny(\prof{J}')$
    and for \textbf{all} $J^{*}_{\mtext{old}} \in \Kemeny(\prof{J})$
    it holds that $d(J^{*}_{\mtext{new}},J_0,w) < d(J^{*}_{\mtext{old}},J_0,w)$?
}

\probdef{
  \OptimisticBribery{\Kemeny}
  
  \emph{Instance:} An agenda~$\Phi$ with an
    integrity constraint~$\Gamma$,
    a weight function~$w : \Pre{\Phi} \rightarrow \NN$,
    a profile~$\prof{J} \in \JJJ(\Phi,\Gamma)^{+}$,
    a judgment set~$J_0 \in \JJJ(\Phi,\Gamma)$,
    and an integer~$k \in \NN$.
  
  \emph{Question:} Is it possible to change up to~$k$ individual
    judgment sets in~$\prof{J}$, resulting in a new profile~$\prof{J}'$,
    so that there is \textbf{some} $J^{*}_{\mtext{new}} \in \Kemeny(\prof{J}')$
    such that for \textbf{all} $J^{*}_{\mtext{old}} \in \Kemeny(\prof{J})$
    it holds that $d(J^{*}_{\mtext{new}},J_0,w) < d(J^{*}_{\mtext{old}},J_0,w)$?
}

\probdef{
  \PessimisticBribery{\Kemeny}
  
  \emph{Instance:} An agenda~$\Phi$ with an
    integrity constraint~$\Gamma$,
    a weight function~$w : \Pre{\Phi} \rightarrow \NN$,
    a profile~$\prof{J} \in \JJJ(\Phi,\Gamma)^{+}$,
    a judgment set~$J_0 \in \JJJ(\Phi,\Gamma)$,
    and an integer~$k \in \NN$.
  
  \emph{Question:} Is it possible to change up to~$k$ individual
    judgment sets in~$\prof{J}$, resulting in a new profile~$\prof{J}'$,
    so that for \textbf{all} $J^{*}_{\mtext{new}} \in \Kemeny(\prof{J}')$
    there is \textbf{some} $J^{*}_{\mtext{old}} \in \Kemeny(\prof{J})$
    such that $d(J^{*}_{\mtext{new}},J_0,w) < d(J^{*}_{\mtext{old}},J_0,w)$?
}

\probdef{
  \SuperoptimisticBribery{\Kemeny}
  
  \emph{Instance:} An agenda~$\Phi$ with an
    integrity constraint~$\Gamma$,
    a weight function~$w : \Pre{\Phi} \rightarrow \NN$,
    a profile~$\prof{J} \in \JJJ(\Phi,\Gamma)^{+}$,
    a judgment set~$J_0 \in \JJJ(\Phi,\Gamma)$,
    and an integer~$k \in \NN$.
  
  \emph{Question:} Is it possible to change up to~$k$ individual
    judgment sets in~$\prof{J}$, resulting in a new profile~$\prof{J}'$,
    so that there is \textbf{some} $J^{*}_{\mtext{new}} \in \Kemeny(\prof{J}')$
    and \textbf{some} $J^{*}_{\mtext{old}} \in \Kemeny(\prof{J})$
    such that $d(J^{*}_{\mtext{new}},J_0,w) < d(J^{*}_{\mtext{old}},J_0,w)$?
}

\probdef{
  \SafeBribery{\Kemeny}
  
  \emph{Instance:} An agenda~$\Phi$ with an
    integrity constraint~$\Gamma$,
    a weight function~$w : \Pre{\Phi} \rightarrow \NN$,
    a profile~$\prof{J} \in \JJJ(\Phi,\Gamma)^{+}$,
    a judgment set~$J_0 \in \JJJ(\Phi,\Gamma)$,
    and an integer~$k \in \NN$.
  
  \emph{Question:} Is it possible to change up to~$k$ individual
    judgment sets in~$\prof{J}$, resulting in a new profile~$\prof{J}'$,
    so that (1)~for \textbf{all} $J^{*}_{\mtext{new}} \in \Kemeny(\prof{J}')$
    and for \textbf{all} $J^{*}_{\mtext{old}} \in \Kemeny(\prof{J})$
    it holds that $d(J^{*}_{\mtext{new}},J_0,w) \leq d(J^{*}_{\mtext{old}},J_0,w)$,
    and so that (2)~there exists
    \textbf{some} $J^{*}_{\mtext{new}} \in \Kemeny(\prof{J}')$
    and \textbf{some} $J^{*}_{\mtext{old}} \in \Kemeny(\prof{J})$
    such that $d(J^{*}_{\mtext{new}},J_0,w) < d(J^{*}_{\mtext{old}},J_0,w)$?
}

\subsection{Complexity Results}

In this section, we prove the following result.

\begin{theorem}
\label{thm:bribery-theorem}
The following problems are \SigmaP{2}-complete:
\begin{itemize}
  \item \CautiousBribery{\Kemeny},
  \item \OptimisticBribery{\Kemeny},
  \item \PessimisticBribery{\Kemeny},
  \item \SuperoptimisticBribery{\Kemeny}, and
  \item \SafeBribery{\Kemeny}.
\end{itemize}
\end{theorem}

This result follows from Propositions~\ref{prop:bribery-cautious-membership}
and~\ref{prop:bribery-cautious-hardness}
and Corollaries~\ref{cor:bribery-others-membership}
and~\ref{cor:bribery-others-hardness},
that we establish below.

\begin{proposition}
\label{prop:bribery-cautious-membership}
\CautiousBribery{\Kemeny} is in \SigmaP{2}.
\end{proposition}
\begin{proof}
We describe a nondeterministic polynomial-time
algorithm with access to an \NP{} oracle
that solves the problem.
This algorithm is similar to the algorithm used in the proof of
Proposition~\ref{prop:manipulation-cautious-membership}.
Let~$(\Phi,\Gamma,w,\prof{J},J_0,k)$ specify an instance of \CautiousBribery{\Kemeny},
where~$\prof{J} = (J_1,\dotsc,J_p)$.
The algorithm proceeds in several steps.

Firstly,~(1) the algorithm determines the minimum
distance~$d^{\mtext{win}}_{\mtext{old}}$
from~$\prof{J}$ to any complete and consistent judgment
set~$J^{*} \in \JJJ(\Phi,\Gamma)$.
That is,~$d^{\mtext{win}}_{\mtext{old}}$ is the cumulative
unweighted Hamming distance
from the judgments in~$\prof{J}$ to the judgment sets~$J^{*} \in \Kemeny(\prof{J})$.
This can be done in (deterministic) polynomial time
using~$O(\log n)$ queries to an \NP{} oracle.

Then,~(2) the algorithm determines the minimum
distance~$d^{\mtext{min}}_{\mtext{old}}$ (weighted by~$w$)
from~$J_0$ to any judgment set~$J^{*} \in \Kemeny(\prof{J})$,
that is, from~$J_0$ to any complete and consistent judgment
set~$J^{*}$ that has cumulative unweighted
Hamming distance~$d^{\mtext{win}}_{\mtext{old}}$
to the profile~$\prof{J}$.
This can also be done in (deterministic) polynomial time
using an \NP{} oracle.

Next,~(3a) the algorithm guesses~$k$
indices~$1 \leq i_1 < \dotsm < i_k \leq p$
and guesses complete judgment sets~$J'_{1},\dotsc,J'_{k}$
together with truth assignments~$\alpha_1,\dotsc,\alpha_k
: \Var{\Phi,\Gamma} \rightarrow \BB{}$, and it checks
for each~$j \in [k]$ whether~$\alpha_j$ satisfies both~$J'_{j}$
and~$\Gamma$.
This can be done in nondeterministic polynomial time.
The profile~$\prof{J}'$ is obtained from~$\prof{J}$ by replacing
the judgment set~$J_{i_j}$ by~$J'_j$, for each~$j \in [k]$.
Moreover,~(3b) the algorithm determines the minimum
distance~$d^{\mtext{win}}_{\mtext{new}}$
from~$(\prof{J}')$ to any complete and consistent judgment
set~$J^{*} \in \JJJ(\Phi,\Gamma)$.
Finally,~(3c) the algorithm determines by using a single query to
an \NP{} oracle whether there exists some complete and consistent
judgment set~$J^{*}_{\mtext{new}} \in \JJJ(\Phi,\Gamma)$ such
that~$d(J^{*}_{\mtext{new}},(\prof{J}')) = d^{\mtext{win}}_{\mtext{new}}$
and~$d(J^{*}_{\mtext{new}},J_0,w) \geq d^{\mtext{min}}_{\mtext{old}}$.
If this is the case, the algorithm rejects;
otherwise, the algorithm accepts.

It is straightforward to verify that the algorithm runs in nondeterministic
polynomial time.
Moreover, the algorithm accepts the input (for some sequence of
nondeterministic choices) if and only if there is some way of replacing
up to~$k$ judgment sets in~$\prof{J}$, resulting in a new profile~$\prof{J}'$,
such that
for all~$J^{*}_{\mtext{new}} \in \Kemeny(\prof{J}')$
and for all~$J^{*}_{\mtext{old}} \in \Kemeny(\prof{J})$
it holds that~$d(J^{*}_{\mtext{new}},J_0,w) < d(J^{*}_{\mtext{old}},J_0,w)$.
\end{proof}

\begin{corollary}
\label{cor:bribery-others-membership}
\OptimisticBribery{\Kemeny},
\PessimisticBribery{\Kemeny},
\SuperoptimisticBribery{\Kemeny}, and
\SafeBribery{\Kemeny}
are in \SigmaP{2}.
\end{corollary}
\begin{proof}[Proof (sketch)]
For each of these problems,
a nondeterministic polynomial-time algorithm with access to an \NP{} oracle can
be constructed that solves the problem.
These algorithms are analogous to the algorithms described in the
proofs of Propositions~\ref{prop:manipulation-cautious-membership}--%
\ref{prop:manipulation-superoptimistic-membership}
and~\ref{prop:bribery-cautious-membership}, and
Corollary~\ref{cor:manipulation-safe-membership}.
\end{proof}

\begin{proposition}
\label{prop:bribery-cautious-hardness}
\CautiousBribery{\Kemeny} is \SigmaP{2}-hard.
\end{proposition}
\begin{proof}
We show \SigmaP{2}-hardness by giving a reduction from the satisfiability
problem for quantified Boolean formulas of the form~$\exists x_1,\dotsc,x_n.
\forall y_1,\dotsc,y_m. \psi$.
Let~$\varphi = \exists x_1,\dotsc,x_n. \forall y_1,\dotsc,y_m. \psi$ be a
quantified Boolean formula.
Let~$X = \SBs x_1,\dotsc,x_n \SEs$ and~$Y = \SBs y_1,\dotsc,y_m \SEs$.
We construct an agenda~$\Phi$, an integrity constraint~$\Gamma$,
a weight function~$w : \Pre{\Phi} \rightarrow \NN$,
a profile~$\prof{J}$ as follows,
a judgment set~$J_0$ and an integer~$k$ as follows.

We consider the variables~$x_1,\dotsc,x_n$ and~$y_1,\dotsc,y_m$
and we introduce fresh variables~$x'_1,\dotsc,x'_n$ and~$y'_1,\dotsc,y'_m$.
Moreover, we introduce fresh variables~$z_i$ for~$i \in [n]$,
fresh variables~$t_j$ for~$j \in [m]$,
fresh variables~$u_{i,\ell}$ for~$i \in [3]$ and~$\ell \in [u]$,
where~$u = 10n+10m+10$.
Finally, we introduce a fresh variable~$a$,
and fresh variables~$b_{j}$ for~$j \in [2m+4]$.
We then let~$\Pre{\Phi} = \SB x_i,x'_i,z_i \SM i \in [n] \SE \cup
\SB y_j,y'_j,t_j \SM j \in [m] \SE \cup \SB u_{i,j} \SM i \in [3], j \in [u] \SE
\cup \SBs a,b_1,\dotsc,b_{2m+4} \SEs$.

We define the integrity constraint~$\Gamma$ as follows.
We let
\[ \Gamma = \Gamma_0 \vee
\bigvee\limits_{i \in [3]} \bigwedge\limits_{\ell \in \mathrlap{[u]}} u_{i,\ell},\]
and
\begin{align*}
  \Gamma_0 =\ &
    \left ( a \vee \bigwedge\limits_{j \in [2m+4]} b_j \right ) \wedge
    \left ( a \rightarrow \bigwedge\limits_{j \in [2m+4]} \neg b_j \right ) \wedge
    \left ( \left ( \bigwedge\limits_{j \in [2m+4]} b_j \right ) \rightarrow \neg a \right ) \\
    \wedge\ & \left (
      \left ( \bigwedge\limits_{i \in [n]} (x_i \oplus x'_i) \right )
    \vee
      \bigwedge\limits_{i \in [n]} z_i
    \right )
      \wedge
    \left (
      \left ( \bigwedge\limits_{i \in [n]} (x_i \oplus x'_i) \right )
    \rightarrow
      \bigwedge\limits_{i \in [n]} \neg z_i
    \right ) \\
  \wedge\ &
    \left (
      \left ( \bigwedge\limits_{j \in [m]} (y_j \oplus y'_j) \right )
    \oplus
      \bigwedge\limits_{j \in [m]} t_j
    \right )
      \wedge
    \left (
      \left ( \bigwedge\limits_{j \in [m]} (y_j \oplus y'_j) \right )
    \rightarrow
      \bigwedge\limits_{j \in [m]} \neg t_j
    \right ) \\
    \wedge\ & \left (
      \left ( \bigwedge\limits_{i \in [n]} z_i \right )
    \rightarrow
      \left ( \bigwedge\limits_{i \in [n]} (\neg x_i \wedge \neg x'_i) \right )
    \right )
      \wedge
    \left (
      \left ( \bigwedge\limits_{j \in [m]} t_j \right )
    \rightarrow
      \bigwedge\limits_{j \in [m]} (\neg y_j \wedge \neg y'_j)
    \right ) \\
  \wedge\ & \left ( \left ( \bigwedge\limits_{i \in [n]} z_i \right ) \rightarrow \left ( a \wedge \bigwedge\limits_{j \in [m]} t_j \right ) \right ) \\
  \wedge\ & \left ( \left ( \bigwedge\limits_{i \in [n]} (x_i \oplus x'_i) \wedge \bigwedge\limits_{j \in [m]} (y_j \oplus y'_j) \right ) \rightarrow (\neg \psi \wedge a) \right ) \\
  \wedge\ & \left ( \left ( \bigwedge\limits_{i \in [n]} (x_i \oplus x'_i) \wedge \bigwedge\limits_{j \in [m]} t_j \right ) \rightarrow \bigwedge\limits_{j \in [2m+4]} b_j \right ).
\end{align*}
(Here~$\oplus$ denotes exclusive disjunction.)

As a result of the definition of~$\Gamma$,
each complete and consistent judgment set~$J \in \JJJ(\Phi,\Gamma)$
satisfies (at least) one of the following four conditions.
\begin{enumerate}
  \item For some~$i \in [3]$, the judgment set~$J$ includes
    each formula~$u_{i,\ell}$ for~$\ell \in [u]$.
  \item The judgment set~$J$ includes
    exactly one of~$x_i$ and~$x'_i$ for each~$i \in [n]$,
    it includes none of the formulas~$z_i$,
    it includes exactly one of~$y_j$ and~$y'_j$ for each~$j \in [m]$,
    it includes none of the formulas~$t_j$,
    it includes~$a$,
    it includes none of the formulas~$b_j$,
    and~$J$ does not satisfy~$\psi$.
  \item The judgment set~$J$ includes
    exactly one of~$x_i$ and~$x'_i$ for each~$i \in [n]$,
    it includes none of the formulas~$z_i$,
    it includes all formulas~$t_j$,
    for each~$j \in [m]$ it includes either none or both of~$y_j$ and~$y'_j$,
    it does not include~$a$,
    and it includes all formulas~$b_j$ for~$j \in [2m+4]$.
  \item The judgment set~$J$ includes
    all formulas~$z_i$,
    for each~$i \in [n]$ it includes either none or both of~$x_i$ and~$x'_i$,
    it includes all formulas~$t_j$,
    for each~$j \in [m]$ it includes either none or both of~$y_j$ and~$y'_j$,
    it includes~$a$, and
    it includes none of the formulas~$b_j$.
\end{enumerate}

We let~$\prof{J} = (J_1,J_2,J_3)$, where~$J_1,J_2,J_3$ are
defined as described in Table~\ref{table:bribery-cautious-hardness-profile}.
In this table, the indices~$i,j,\ell$ range over all possible
values, and for each~$\varphi \in \Pre{\Phi}$ we write a~$0$
if~$\varphi \not\in J_i$ and a~$1$ if~$\varphi \in J_i$.

\begin{table}[ht!]
  \begin{center}
  \begin{tabular}{c | c @{\ \ } c @{\ \ } c @{\ \ } c @{\ \ } c @{\ \ } c @{\ \ } c @{\ \ } c @{\ \ } c @{\ \ } c @{\ \ } c @{\ \ } c}
    \toprule
    $\prof{J}$ & $x_i$ & $x'_i$ & $z_i$ & $y_j$ & $y'_j$ & $t_j$ &
       $a$ & $b_j$ &
       $u_{1,\ell}$ & $u_{2,\ell}$ & $u_{3,\ell}$ \\
    \midrule
    $J_1$ & $0$ & $0$ & $0$ & $0$ & $0$ &
      $0$ & $0$ & $0$ & $1$ & $0$ & $0$ \\
    $J_2$ & $0$ & $0$ & $0$ & $0$ & $0$ &
      $0$ & $0$ & $0$ & $0$ & $1$ & $0$ \\
    $J_3$ & $0$ & $0$ & $0$ & $0$ & $0$ &
      $0$ & $0$ & $0$ & $0$ & $0$ & $1$ \\
    \bottomrule
  \end{tabular}
  \end{center}
  \caption{The profile~$\prof{J} = (J_1,J_2,J_3)$ that we use in the
    proof of Proposition~\ref{prop:bribery-cautious-hardness}.}
  \label{table:bribery-cautious-hardness-profile}
\end{table}

Moreover, we let~$J_0$ be
as described in Table~\ref{table:bribery-cautious-hardness-desired-opinion}.
In this table, the indices~$i,j,\ell$ range over all possible
values, and for each~$\varphi \in \Pre{\Phi}$ we write a~$0$
if~$\varphi \not\in J_i$ and a~$1$ if~$\varphi \in J_i$.
In addition, we define the weight function~$w$ as follows.
For each~$j \in [2m+4]$, we let~$w(b_j) = 1$,
and for all other formulas~$\varphi \in \Pre{\Phi}$, we let~$w(\varphi) = 0$.
In other words, the briber wants the formulas~$b_j$ to be satisfied,
and does not care about any other formula.

\begin{table}[ht!]
  \begin{center}
  \begin{tabular}{c | c @{\ \ } c @{\ \ } c @{\ \ } c @{\ \ } c @{\ \ } c @{\ \ } c @{\ \ } c @{\ \ } c @{\ \ } c @{\ \ } c @{\ \ } c}
    \toprule
    & $x_i$ & $x'_i$ & $z_i$ & $y_j$ & $y'_j$ & $t_j$ &
       $a$ & $b_j$ &
       $u_{1,\ell}$ & $u_{2,\ell}$ & $u_{3,\ell}$ \\
    \midrule
    $J_0$ & $0$ & $0$ & $0$ & $0$ & $0$ &
      $0$ & $0$ & $1$ & $1$ & $0$ & $0$ \\
    \bottomrule
  \end{tabular}
  \end{center}
  \caption{The judgment set~$J_0$ that we use in the
    proof of Proposition~\ref{prop:bribery-cautious-hardness}.}
  \label{table:bribery-cautious-hardness-desired-opinion}
\end{table}

We let~$k = 1$. That is, the briber can bribe exactly one individual.

In the remainder,
we will argue that there is some truth assignment~$\alpha : X \rightarrow \BB{}$
such that for all truth assignments~$\beta : Y \rightarrow \BB{}$
it holds that~$\psi[\alpha \cup \beta]$ is true
if and only if
it is possible to change up to~$k$ individual
judgment sets in~$\prof{J}$, resulting in a new profile~$\prof{J}'$,
so that for {all} $J^{*}_{\mtext{new}} \in \Kemeny(\prof{J}')$
and for {all} $J^{*}_{\mtext{old}} \in \Kemeny(\prof{J})$
it holds that $d(J^{*}_{\mtext{new}},J_0,w) < d(J^{*}_{\mtext{old}},J_0,w)$.

Firstly, we observe that~$\Kemeny(\prof{J}) = \SBs J^{*}_{\mtext{old}} \SEs$,
where~$J^{*}_{\mtext{old}}$ is defined as described
in Table~\ref{table:bribery-cautious-hardness-old-winners}.
The judgment set~$J^{*}_{\mtext{old}}$ is complete
and consistent, and has a cumulative Hamming distance of~$6n+3+3u$
to the profile~$\prof{J}$.
It is straightforward to verify that no complete and consistent judgment set
has a smaller cumulative Hamming distance to the profile~$\prof{J}$.

\begin{table}[ht!]
  \begin{center}
  \begin{tabular}{c | c @{\ \ } c @{\ \ } c @{\ \ } c @{\ \ } c @{\ \ } c @{\ \ } c @{\ \ } c @{\ \ } c @{\ \ } c @{\ \ } c @{\ \ } c}
    \toprule
    $\Kemeny(\prof{J})$ & $x_i$ & $x'_i$ & $z_i$ & $y_j$ & $y'_j$ & $t_j$ &
       $a$ & $b_j$ &
       $u_{1,\ell}$ & $u_{2,\ell}$ & $u_{3,\ell}$ \\
    \midrule
    $J^{*}_{\mtext{old}}$ & $0$ & $0$ & $1$ & $0$ & $0$ & $1$ &
      $1$ & $0$ & $0$ & $0$ & $0$ \\
    \bottomrule
  \end{tabular}
  \end{center}
  \caption{The judgment set~$J^{*}_{\mtext{old}}$ that is used in the
    proof of Proposition~\ref{prop:bribery-cautious-hardness}.}
  \label{table:bribery-cautious-hardness-old-winners}
\end{table}

We observe that the only way for the briber to enforce an
outcome that satisfies the formulas~$b_j$, is to enforce an
outcome~$J^{*}_{\mtext{new}}$
that satisfies condition~(3).
Additionally, we know that any such enforced
outcome~$J^{*}_{\mtext{new}}$ includes none of the
formulas~$y_j$ and~$y'_j$, as a majority of judgment sets
in the (modified) profile includes~$\neg y_j$ and~$\neg y'_j$.
Similarly, any such enforced
outcome~$J^{*}_{\mtext{new}}$ includes none of the
formulas~$u_{i,\ell}$.

We argue that there is some truth assignment~$\alpha : X \rightarrow \BB{}$
such that for all truth assignments~$\beta : Y \rightarrow \BB{}$
it holds that~$\psi[\alpha \cup \beta]$ is true
if and only if
it is possible to change a single
judgment set in~$\prof{J}$, resulting in a new profile~$\prof{J}'$,
so that for {all} $J^{*}_{\mtext{new}} \in \Kemeny(\prof{J}')$
and for {all} $J^{*}_{\mtext{old}} \in \Kemeny(\prof{J})$
it holds that $d(J^{*}_{\mtext{new}},J_0,w) < d(J^{*}_{\mtext{old}},J_0,w)$.

$(\Rightarrow)$
Suppose that there exists some truth assignment~$\alpha : X \rightarrow \BB{}$
such that for all truth assignments~$\beta : Y \rightarrow \BB{}$
it holds that~$\psi[\alpha \cup \beta]$ is true.
Consider the complete and consistent judgment set~$J^{*}_{\alpha}$
that is described in Table~\ref{table:bribery-cautious-hardness-insincere-judgment}.
Moreover, let~$\prof{J}'$ be the profile obtained from~$\prof{J}$ by replacing
any single judgment set in~$\prof{J}$ by~$J^{*}_{\alpha}$
(since the judgment sets in~$\prof{J}$ are symmetric, the argument
does not depend on which set is replaced).
We then get that~$\Kemeny(\prof{J}') = \SBs J^{*}_{\alpha} \SEs$.
The only possible complete and consistent judgment sets~$J^{*}$
that could have a smaller Hamming distance to the
profile~$(\prof{J}')$ would have to satisfy condition~(2).
That is, such judgment sets~$J^{*}$ would have to include exactly
one of~$y_j$ and~$y'_j$, for each~$j \in [m]$,
and would have to satisfy~$\neg\psi$.
Moreover, in order for such judgment sets~$J^{*}$ to have a smaller
Hamming distance to the profile, it would have to agree with~$J^{*}_{\alpha}$
on the formulas~$x_i$ and~$x'_i$, for all~$i \in [n]$.
In particular, the cumulative unweighted Hamming distance from~$\prof{J}'$
to~$J^{*}_{\alpha}$ is~$2n+6m+8+2u$,
and the cumulative unweighted Hamming distance from~$\prof{J}'$
to such a judgment set~$J^{*}$ would be~$2n+6m+7+2u$.
However, since~$\psi[\alpha]$ is valid, we know that such judgment
sets~$J^{*}$ are not consistent.
Therefore,~$\Kemeny(\prof{J}') = \SBs J^{*}_{\alpha} \SEs$.
Clearly,~$d(J_0,J^{*}_{\alpha},w) < d(J_0,J^{*}_{\mtext{old}},w)$.
In other words,
for all~$J^{*}_{\mtext{new}} \in \Kemeny(\prof{J}')$
and for all~$J^{*}_{\mtext{old}} \in \Kemeny(\prof{J})$
it holds that $d(J^{*}_{\mtext{new}},J_0,w) < d(J^{*}_{\mtext{old}},J_0,w)$.

\begin{table}[ht!]
  \begin{center}
  \begin{tabular}{c | c @{\ \ } c @{\ \ } c @{\ \ } c @{\ \ } c @{\ \ } c @{\ \ } c @{\ \ } c @{\ \ } c @{\ \ } c @{\ \ } c @{\ \ } c}
    \toprule
     & $x_i$ & $x'_i$ & $z_i$ & $y_j$ & $y'_j$ & $t_j$ &
       $a$ & $b_j$ &
       $u_{1,\ell}$ & $u_{2,\ell}$ & $u_{3,\ell}$ \\
    \midrule
    $J^{*}_{\alpha}$ & $\alpha(x_i)$ & $1-\alpha(x_i)$ & $0$ & $0$ & $0$ & $1$ &
      $0$ & $1$ & $0$ & $0$ & $0$ \\
    \bottomrule
  \end{tabular}
  \end{center}
  \caption{The judgment set~$J^{*}_{\alpha}$ that is used in
    proof of Proposition~\ref{prop:bribery-cautious-hardness}.}
  \label{table:bribery-cautious-hardness-insincere-judgment}
\end{table}

$(\Leftarrow)$
Conversely, suppose that it is possible to change a single
judgment set in~$\prof{J}$, resulting in a new profile~$\prof{J}'$,
so that for {all} $J^{*}_{\mtext{new}} \in \Kemeny(\prof{J}')$
and for {all} $J^{*}_{\mtext{old}} \in \Kemeny(\prof{J})$
it holds that $d(J^{*}_{\mtext{new}},J_0,w) < d(J^{*}_{\mtext{old}},J_0,w)$.
As we argued above, each such complete and consistent
judgment set~$J^{*}_{\mtext{new}}$ satisfies condition~(3),
and furthermore, it does not contain any of the
formulas~$y_j,y'_j$ and~$u_{i,\ell}$.
That is, such a judgment set~$J^{*}_{\mtext{new}}$
must be of the form~$J^{*}_{\alpha}$,
for some~$\alpha : X \rightarrow \BB{}$,
as described in
Figure~\ref{table:bribery-cautious-hardness-insincere-judgment}.
Fix this truth assignment~$\alpha : X \rightarrow \BB{}$.
It suffices to consider profiles~$\prof{J}'$ that are obtained from
the profile~$\prof{J}$ by replacing one of the judgment sets
by~$J^{*}_{\alpha}$, as replacing a judgment set in~$\prof{J}$
with a judgment set that differs from~$J^{*}_{\alpha}$ can only
increase the cumulative unweighted Hamming distance to~$J^{*}_{\alpha}$.

\begin{table}[ht!]
  \begin{center}
  \begin{tabular}{c | c @{\ \ } c @{\ \ } c @{\ \ } c @{\ \ } c @{\ \ } c @{\ \ } c @{\ \ } c @{\ \ } c @{\ \ } c @{\ \ } c @{\ \ } c}
    \toprule
     & $x_i$ & $x'_i$ & $z_i$ & $y_j$ & $y'_j$ & $t_j$ &
       $a$ & $b_j$ &
       $u_{1,\ell}$ & $u_{2,\ell}$ & $u_{3,\ell}$ \\
    \midrule
    $J^{*}_{\beta}$ & $\alpha(x_i)$ & $1-\alpha(x_i)$ & $0$ & $\beta(y_j)$ & $1-\beta(y_j)$ & $0$ &
      $1$ & $0$ & $0$ & $0$ & $0$ \\
    \bottomrule
  \end{tabular}
  \end{center}
  \caption{The judgment set~$J^{*}_{\beta}$ that is used in
    the proof of Proposition~\ref{prop:bribery-cautious-hardness}.}
  \label{table:bribery-cautious-hardness-counterexample}
\end{table}

We now use the fact that~$J^{*}_{\alpha} \in \Kemeny(\prof{J}')$ to argue
that for all truth assignments~$\beta : Y \rightarrow \BB{}$
it holds that~$\psi[\alpha \cup \beta]$ is true.
We proceed indirectly, and assume that there exists some truth
assignment~$\beta : Y \rightarrow \BB{}$
such that~$\psi[\alpha \cup \beta]$ is false.
Then consider the complete and consistent judgment set~$J^{*}_{\beta}$
that is described in Table~\ref{table:bribery-cautious-hardness-counterexample}.
Because~$\psi[\alpha \cup \beta]$ is false, we know that~$J^{*}_{\beta}$
is consistent---it satisfies condition~(2).
Moreover,~$d(J^{*}_{\beta},\prof{J}') < d(J^{*}_{\alpha},\prof{J}')$.
Namely, the judgment set~$J^{*}_{\alpha}$ has Hamming distance~$2n+6m+8+2u$
to the profile~$\prof{J}'$,
and the judgment set~$J^{*}_{\beta}$ has Hamming distance~$2n+6m+7+2u$
to the profile~$\prof{J}'$.
This is a contradiction with the fact that~$J^{*}_{\alpha} \in \Kemeny(\prof{J}')$.
Therefore, we can conclude that no truth
assignment~$\beta : Y \rightarrow \BB{}$ exists
such that~$\psi[\alpha \cup \beta]$ is false.
In other words, for all truth assignments~$\beta : Y \rightarrow \BB{}$
it holds that~$\psi[\alpha \cup \beta]$ is true.
\end{proof}

\begin{corollary}
\label{cor:bribery-others-hardness}
\OptimisticBribery{\Kemeny},
\PessimisticBribery{\Kemeny},
\SuperoptimisticBribery{\Kemeny},
and \SafeBribery{\Kemeny}
are \SigmaP{2}-hard.
\end{corollary}
\begin{proof}
\SigmaP{2}-hardness for \OptimisticBribery{\Kemeny},
\PessimisticBribery{\Kemeny},
\SuperoptimisticBribery{\Kemeny}, and \SafeBribery{\Kemeny}
follows from the proof of Proposition~\ref{prop:bribery-cautious-hardness}.
In this proof, all Kemeny outcomes for the original profile~$\prof{J}$
have the same weighted Hamming distance to the judgment set~$J_0$.
Similarly, all Kemeny outcomes for the relevant modified
profiles~$\prof{J}'$,
have the same weighted Hamming distance to the judgment set~$J_0$.
Therefore, the reduction can also be used to show \SigmaP{2}-hardness
for the problems \OptimisticBribery{\Kemeny},
\PessimisticBribery{\Kemeny},
\SuperoptimisticBribery{\Kemeny}, and \SafeBribery{\Kemeny}.
\end{proof}

\subsection{Restricted Settings and Constraint-Based Judgment Aggregation}

Similarly to the case of Theorem~\ref{thm:manipulation-theorem},
the result of Theorem~\ref{thm:bribery-theorem} can straightforwardly be
extended to the setting where~$\Gamma = \top$ and where all
formulas~$\varphi \in \Pre{\Phi}$ are clauses of size~$3$,
by arguments that are entirely similar to the ones described
in Section~\ref{sec:manipulation-3clauses-trivial-ic}.

The result of Theorem~\ref{thm:bribery-theorem} can also straightforwardly be
extended to the setting of constraint-based judgment aggregation.
The algorithms used in the proofs of
Proposition~\ref{prop:bribery-cautious-membership} and
Corollary~\ref{cor:bribery-others-membership}
can directly be used in the
setting of constraint-based judgment aggregation as well
(so membership in \SigmaP{2} carries over to this setting).
Moreover,
since the proofs of Proposition~\ref{prop:bribery-cautious-hardness}
and Corollary~\ref{cor:bribery-others-hardness} use an
agenda containing only propositional variables (and their negations),
this proof can also directly be used to show \SigmaP{2}-hardness for the
setting of constraint-based judgment aggregation.

\subsection{Exact Bribery}

In the literature, the problem of bribery in judgment aggregation has also been
modelled using different decision problems (see,
e.g.,~\cite{BaumeisterErdelyiErdelyiRothe15}).
One of these decision problems asks whether the briber has a way of bribing
a number of individuals (within their budget) to obtain an outcome that
includes a given desired subset of the agenda.
This problem is often called \emph{exact bribery}.
We consider two variants of this exact bribery problem,
and we argue that the \SigmaP{2}-completeness result of
Theorem~\ref{thm:bribery-theorem} extends to these problems.

\probdef{
  \CautiousExactBribery{\Kemeny}
  
  \emph{Instance:} An agenda~$\Phi$ with an
    integrity constraint~$\Gamma$,
    a profile~$\prof{J} \in \JJJ(\Phi,\Gamma)^{+}$,
    a (possibly incomplete) judgment set~$L \subseteq \Phi$,
    and an integer~$k \in \NN$.
  
  \emph{Question:} Is it possible to change up to~$k$ individual
    judgment sets in~$\prof{J}$, resulting in a new profile~$\prof{J}'$,
    so that for \textbf{all} $J^{*} \in \Kemeny(\prof{J}')$
    it holds that $L \subseteq J^{*}$?
}

\probdef{
  \BraveExactBribery{\Kemeny}
  
  \emph{Instance:} An agenda~$\Phi$ with an
    integrity constraint~$\Gamma$,
    a profile~$\prof{J} \in \JJJ(\Phi,\Gamma)^{+}$,
    a (possibly incomplete) judgment set~$L \subseteq \Phi$,
    and an integer~$k \in \NN$.
  
  \emph{Question:} Is it possible to change up to~$k$ individual
    judgment sets in~$\prof{J}$, resulting in a new profile~$\prof{J}'$,
    so that for \textbf{some} $J^{*} \in \Kemeny(\prof{J}')$
    it holds that $L \subseteq J^{*}$?
}

\begin{proposition}
\label{prop:exact-bribery}
The problems \CautiousExactBribery{\Kemeny} and
\BraveExactBribery{\Kemeny} are \SigmaP{2}-complete.
Moreover, \SigmaP{2}-hardness holds even for the case
where~$\Card{L} = 1$.
\end{proposition}
\begin{proof}
To show membership in \SigmaP{2}, it suffices to observe that
the algorithms used in the proofs of
Proposition~\ref{prop:bribery-cautious-membership} and
Corollary~\ref{cor:bribery-others-membership} can straightforwardly
be extended to the case of \CautiousExactBribery{\Kemeny} and
\BraveExactBribery{\Kemeny}.
For \SigmaP{2}-hardness, one can use the proofs of
Proposition~\ref{prop:bribery-cautious-hardness}
and Corollary~\ref{cor:bribery-others-hardness},
with the modification that~$L = \SBs b_1 \SEs$.
\end{proof}

\section{Control by Adding or Deleting Issues}
\label{sec:control}

A third form of strategic behavior in judgment aggregation is
control.
In this setting, an external agent wishes to influence the outcome of
by influencing the conditions of a judgment aggregation scenario.
Here, we consider control by (1)~adding or (2)~deleting issues.
We model these scenarios as follows.
(Again, we consider the formula-based judgment aggregation framework;
for the constraint-based judgment aggregation framework, this
control scenario can be defined entirely similarly.)

We begin with the scenario of (1)~control by adding issues.
A number of individuals each have an opinion for
an agenda~$\Phi$ in the presence of an integrity constraint~$\Gamma$.
That is, we are considering a profile~$\prof{J} \in \JJJ(\Phi,\Gamma)^{+}$.
However, they are performing judgment aggregation only on
a selection of issues, specified by an agenda~$\Phi' \subseteq \Phi$.
(For any~$\Psi \subseteq \Phi$, we let the profile~$\prof{J}|_{\Psi}$
consist of the judgment sets~$J|_{\Psi}$ for each~$J \in \prof{J}$,
where~$J|_{\Psi} = J \cap \Psi$---%
that is,~$\prof{J}|_{\Psi} = \SB J \cap \Psi \SM J \in \prof{J} \SE$.
Intuitively,~$\prof{J}|_{\Psi}$ is the restriction of~$\prof{J}$
to the formulas in~$\Psi$.)
The external agent wishes to ensure that the outcome
of the judgment aggregation procedure includes a set~$L \subseteq \Phi$
of conclusions, and they want to do so by enlarging the set of issues
that the individuals perform judgment aggregation on.
Formally, the external agent wants to select an agenda~$\Phi''$
with~$\Phi' \subseteq \Phi'' \subseteq \Phi$ such
that~$L \subseteq J^{*}$ for~$J^{*} \in \Kemeny(\prof{J}|_{\Phi''})$.
(Obviously, if the external agent wishes to succeed, they need to
choose some~$\Phi''$ with~$L \subseteq \Phi''$.)
The question is whether the external agent can succeed in this.
Since the Kemeny judgment aggregation procedure is irresolute,
we can formulate two decision problems.

\probdef{
  \CautiousExactControlByAddingIssues{\Kemeny}
  
  \emph{Instance:} An agenda~$\Phi$ with an
    integrity constraint~$\Gamma$,
    an agenda~$\Phi' \subseteq \Phi$,
    a set~$L \subseteq \Phi$
    and a profile~$\prof{J} \in \JJJ(\Phi,\Gamma)^{+}$ for~$\Phi$.
  
  \emph{Question:} Is there an agenda~$\Phi' \subseteq \Phi'' \subseteq \Phi$
    such that for \textbf{all}~$J^{*} \in \Kemeny(\prof{J}|_{\Phi''})$ it holds
    that~$L \subseteq J^{*}$?
}

\probdef{
  \BraveExactControlByAddingIssues{\Kemeny}
  
  \emph{Instance:} An agenda~$\Phi$ with an
    integrity constraint~$\Gamma$,
    an agenda~$\Phi' \subseteq \Phi$,
    a set~$L \subseteq \Phi$
    and a profile~$\prof{J} \in \JJJ(\Phi,\Gamma)^{+}$ for~$\Phi$.
  
  \emph{Question:} Is there an agenda~$\Phi' \subseteq \Phi'' \subseteq \Phi$
    such that for \textbf{some}~$J^{*} \in \Kemeny(\prof{J}|_{\Phi''})$ it holds
    that~$L \subseteq J^{*}$?
}

We continue with the scenario of (2)~control by deleting issues.
In this scenario, the individuals are performing judgment
aggregation on an agenda~$\Phi$ in the presence of an integrity
constraint~$\Gamma$.
The external agent wishes to ensure that the outcome
of the judgment aggregation procedure includes a set~$L \subseteq \Phi$
of conclusions, and they want to do so by restricting the set of issues
that the individuals perform judgment aggregation on.
Moreover, several issues are outside the control of the agent---%
these issues cannot be excluded from the judgment aggregation scenario.
This could be the case, for instance, if (1)~the external agent
has influence only on issues regarding one theme,
or if (2)~various issues are politically sensitive
and excluding them would lead to controversy.
This is modelled by a set~$\Phi'' \subseteq \Phi$ of issues
that the external agent cannot exclude.
Formally, the external agent wants to select an agenda~$\Phi''$
with~$\Phi' \subseteq \Phi'' \subseteq \Phi$ such
that~$L \subseteq J^{*}$ for~$J^{*} \in \Kemeny(\prof{J}|_{\Phi''})$.
(Again, obviously, if the external agent wishes to succeed, they need to
choose some~$\Phi''$ with~$L \subseteq \Phi''$.)
The question is whether the external agent can succeed in this.
Again, since the Kemeny judgment aggregation procedure is irresolute,
we can formulate two decision problems.

\probdef{
  \CautiousExactControlByDeletingIssues{\Kemeny}
  
  \emph{Instance:} An agenda~$\Phi$ with an
    integrity constraint~$\Gamma$,
    an agenda~$\Phi' \subseteq \Phi$,
    a set~$L \subseteq \Phi$
    and a profile~$\prof{J} \in \JJJ(\Phi,\Gamma)^{+}$.
  
  \emph{Question:} Is there an agenda~$\Phi' \subseteq \Phi'' \subseteq \Phi$
    such that for \textbf{all}~$J^{*} \in \Kemeny(\prof{J}|_{\Phi'})$ it holds
    that~$L \subseteq J^{*}$?
}

\probdef{
  \BraveExactControlByDeletingIssues{\Kemeny}
  
  \emph{Instance:} An agenda~$\Phi$ with an
    integrity constraint~$\Gamma$,
    an agenda~$\Phi' \subseteq \Phi$,
    a set~$L \subseteq \Phi$
    and a profile~$\prof{J} \in \JJJ(\Phi,\Gamma)^{+}$.
  
  \emph{Question:} Is there an agenda~$\Phi' \subseteq \Phi'' \subseteq \Phi$
    such that for \textbf{some}~$J^{*} \in \Kemeny(\prof{J}|_{\Phi'})$ it holds
    that~$L \subseteq J^{*}$?
}

Notice that the formulations of the decision problems
\CautiousExactControlByAddingIssues{\Kemeny}
and \CautiousExactControlByDeletingIssues{\Kemeny}
are identical.
However, the two different control scenarios that underlie these
problems yield different restrictions of the problems that are
interesting to investigate.
For example, on the one hand,
for the setting of control by deleting issues,
it might often be the case that~$\Phi' = \emptyset$---that is, the
external agent has full control on what issues to delete.
On the other hand, for the setting of control by adding issues,
it is reasonable to assume that~$\Phi' \neq \emptyset$---%
that is, the existing judgment aggregation scenario is not
a trivial, vacuous one.

Unlike for the scenarios of manipulation and bribery,
in this scenario of control we only formulated exact variants of the problem,
where the objective is to obtain outcomes that include a given set~$L$
of conclusions, rather than obtaining outcomes that are preferable
according to a preference relation based on a weighted Hamming distance.
This is due to the fact that in this control scenario, the agenda is not fixed.
As a result, it is unclear what the agenda should be on which the
weighted Hamming distance preferences are based.

\subsection{Complexity Results}

In this section, we prove the following result.

\begin{theorem}
\label{thm:control-theorem}
The following problems are \SigmaP{2}-complete:
\begin{itemize}
  \item \CautiousExactControlByAddingIssues{\Kemeny},
  \item \BraveExactControlByAddingIssues{\Kemeny},
  \item \CautiousExactControlByDeletingIssues{\Kemeny}, and
  \item \BraveExactControlByDeletingIssues{\Kemeny}.
\end{itemize}
Moreover, \SigmaP{2}-hardness holds even for the case
where~$\Phi' = \emptyset$.
\end{theorem}

This result follows from Propositions~\ref{prop:control-addition-cautious-membership}
and~\ref{prop:control-addition-cautious-hardness}
and Corollaries~\ref{cor:control-addition-brave-membership}
and~\ref{cor:control-addition-brave-hardness}--\ref{cor:control-deletion-hardness},
that we establish below.

\begin{proposition}
\label{prop:control-addition-cautious-membership}
\CautiousExactControlByAddingIssues{\Kemeny}
is in \SigmaP{2}.
\end{proposition}
\begin{proof}
We describe a nondeterministic polynomial-time
algorithm with access to an \NP{} oracle
that solves the problem.
Let~$(\Phi,\Gamma,\Phi',L,\prof{J})$ specify an
instance of \CautiousExactControlByAddingIssues{\Kemeny}.
The algorithm guesses an agenda~$\Phi''$ such
that~$\Phi' \subseteq \Phi'' \subseteq \Phi$.
Then, the algorithm computes the minimum unweighted Hamming
distance~$d^{\mtext{win}}$ from any judgment set that is $\Gamma$-consistent
and complete for~$\Phi''$ to the profile~$\prof{J}|_{\Phi''}$.
This can be done in polynomial time using an \NP{} oracle.
Finally, the algorithm uses one more query to the \NP{} oracle to decide
if there exists a judgment set~$J^{*}$ that is $\Gamma$-consistent and
complete for~$\Phi''$ that has Hamming distance~$d^{\mtext{win}}$
to the profile~$\prof{J}|_{\Phi''}$ and that satisfies that~$L \not\subseteq J^{*}$.
The algorithm accepts if and only if no such judgment set~$J^{*}$ exists.

It is straightforward to verify that the algorithm runs in nondeterministic
polynomial time.
Moreover, the algorithm accepts the input (for some sequence of
nondeterministic choices) if and only if
there exists an agenda~$\Phi' \subseteq \Phi'' \subseteq \Phi$
such that for all~$J^{*} \in \Kemeny(\prof{J}|_{\Phi''})$ it holds
that~$L \subseteq J^{*}$.
\end{proof}

\begin{corollary}
\label{cor:control-addition-brave-membership}
\BraveExactControlByAddingIssues{\Kemeny}
is in \SigmaP{2}.
\end{corollary}
\begin{proof}
The algorithm described in the proof of
Proposition~\ref{prop:control-addition-cautious-membership}
can straightforwardly be modified to form
a nondeterministic polynomial-time algorithm with access to an \NP{} oracle
that solves \BraveExactControlByAddingIssues{\Kemeny}.
\end{proof}

\begin{proposition}
\label{prop:control-addition-cautious-hardness}
\CautiousExactControlByAddingIssues{\Kemeny}
is \SigmaP{2}-hard.
\end{proposition}
\begin{proof}
We show \SigmaP{2}-hardness by giving a reduction from the satisfiability
problem for quantified Boolean formulas of the form~$\exists x_1,\dotsc,x_n.
\forall y_1,\dotsc,y_m. \psi$.
Let~$\varphi = \exists x_1,\dotsc,x_n. \forall y_1,\dotsc,y_m. \psi$ be a
quantified Boolean formula.
Let~$X = \SBs x_1,\dotsc,x_n \SEs$ and~$Y = \SBs y_1,\dotsc,y_m \SEs$.
Moreover, assume without loss of generality that~$\psi$ is in 3DNF.
We construct an agenda~$\Phi$, an integrity constraint~$\Gamma$,
an agenda~$\Phi' \subseteq \Phi$,
a set~$L \subseteq \Phi'$
a profile~$\prof{J} \in \JJJ(\Phi,\Gamma)^{+}$ as follows.

We introduce fresh variables~$x_{i,1},x_{i,2},x_{i,3},x'_{i,1},x'_{i,2},x'_{i,3}$
for each~$i \in [n]$,
we consider the variables~$y_j$
and we introduce fresh variables~$y'_j,t_j$ for each~$j \in [m]$,
we introduce fresh variables~$a,b$,
and we introduce fresh variables~$u_{i,\ell}$
for~$i \in [3]$ and~$\ell \in [u]$, where~$u = 10n+10m+10$.
We then let~$\Pre{\Phi'} =
\SB y_j,y'_j,t_j \SM j \in [m] \SE \cup \SBs a,b \SEs \cup
\SB u_{i,\ell} \SM i \in [3], \ell \in [u] \SE$,
and we let~$\Pre{\Phi} = \Pre{\Phi'} \cup
\SB x_{i,j},x'_{i,j} \SM i \in [n], j \in [3] \SE$.
Moreover, we let~$L = \SBs b \SEs$.

We then define the integrity constraint~$\Gamma$ as follows.
We let
\[ \Gamma = \Gamma_0 \vee
\bigvee\limits_{i \in [3]} \bigwedge\limits_{\ell \in \mathrlap{[u]}} u_{i,\ell},\]
and we let
\begin{align*}
  \Gamma_0 =\ &
    (a \vee b) \wedge (\neg a \vee \neg b) \\
  \wedge\ &
    \left (\left ( \bigwedge\limits_{j \in [m]} (y_j \oplus y'_j) \right )
    \vee
    \left ( \bigwedge\limits_{j \in [m]} t_j \right ) \right ) \\
  \wedge\ &
    \left ( \left (
      \bigwedge\limits_{i \in [n]} (x_i \oplus x'_i)
    \right )
    \rightarrow \left (
      \left (
        \bigwedge\limits_{j \in [m]} (y_j \oplus y'_j) \wedge a \wedge \neg\psi
      \right ) \right . \right . \\
  & \hspace{14em} \left . \left .
      \vee
      \left (
        b \wedge \bigwedge\limits_{j \in [m]} (t_j \wedge \neg y_j \wedge \neg y'_j)
      \right )
    \right ) \right ).
\end{align*}
(Here~$\oplus$ denotes exclusive disjunction.)

We let~$\prof{J} = (J_1,J_2,J_3)$, where~$J_1,J_2,J_3$ are
defined as described in Table~\ref{table:control-addition-cautious-hardness-profile}.
In this table, the indices~$i,j,\ell$ range over all possible
values, and for each~$\varphi \in \Pre{\Phi}$ we write a~$0$
if~$\varphi \not\in J_i$ and a~$1$ if~$\varphi \in J_i$.

\begin{table}[ht!]
  \begin{center}
  \begin{tabular}{c | c @{\ \ } c @{\ \ } c @{\ \ } c @{\ \ } c @{\ \ } c @{\ \ } c @{\ \ } c @{\ \ } c @{\ \ } c @{\ \ } c}
    \toprule
    $\prof{J}$ & $x_{i,j}$ & $x'_{i,j}$ & $y_j$ & $y'_j$ & $t_j$ &
       $a$ & $b$ &
       $u_{1,\ell}$ & $u_{2,\ell}$ & $u_{3,\ell}$ \\
    \midrule
    $J_1$ & $1$ & $1$ & $0$ & $0$ & $0$ &
      $0$ & $0$ & $1$ & $0$ & $0$ \\
    $J_2$ & $1$ & $1$ & $0$ & $0$ & $0$ &
      $0$ & $0$ & $0$ & $1$ & $0$ \\
    $J_3$ & $1$ & $1$ & $0$ & $0$ & $0$ &
      $1$ & $0$ & $0$ & $0$ & $1$ \\
    \bottomrule
  \end{tabular}
  \end{center}
  \caption{The profile~$\prof{J} = (J_1,J_2,J_3)$ for the agenda~$\Phi$
    that we use in the
    proof of Proposition~\ref{prop:control-addition-cautious-hardness}.}
  \label{table:control-addition-cautious-hardness-profile}
\end{table}

We observe that~$\Kemeny(\prof{J}|_{\Phi'}) = \SBs J^{*}_{\mtext{old},0} \SEs
\cup \SB J^{*}_{\mtext{old},\beta} \SM \beta : Y \rightarrow \BB \SE$,
where the judgment sets~$J^{*}_{\mtext{old},0}$
and~$J^{*}_{\mtext{old},\beta}$ are defined as described
in Table~\ref{table:control-addition-cautious-hardness-old-winners}.
The judgment sets~$J^{*}_{\mtext{old},0}$ and~$J^{*}_{\mtext{old},\beta}$
are complete (for the agenda~$\Phi'$)
and $\Gamma$-consistent, and have a cumulative Hamming distance of~$3m+2+3u$
to the profile~$\prof{J}|_{\Phi'}$.
It is straightforward to verify that no complete and consistent judgment set
has a smaller cumulative Hamming distance to the profile~$\prof{J}|_{\Phi'}$.

\begin{table}[ht!]
  \begin{center}
  \begin{tabular}{c | c @{\ \ } c @{\ \ } c @{\ \ } c @{\ \ } c @{\ \ } c @{\ \ } c @{\ \ } c @{\ \ } c @{\ \ } c @{\ \ } c}
    \toprule
    $\Kemeny(\prof{J}|_{\Phi'})$ & $y_j$ & $y'_j$ & $t_j$ &
       $a$ & $b$ &
       $u_{1,\ell}$ & $u_{2,\ell}$ & $u_{3,\ell}$ \\
    \midrule
    $J^{*}_{\mtext{old},0}$ & $0$ & $0$ & $1$ &
      $1$ & $0$ & $0$ & $0$ & $0$ \\
    $J^{*}_{\mtext{old},\beta}$ & $\beta(y_j)$ & $1-\beta(y_j)$ & $0$ &
      $1$ & $0$ & $0$ & $0$ & $0$ \\
    \bottomrule
  \end{tabular}
  \end{center}
  \caption{The judgment set~$J^{*}_{\mtext{old}}$ that is used in the
    proof of Proposition~\ref{prop:control-addition-cautious-hardness}.}
  \label{table:control-addition-cautious-hardness-old-winners}
\end{table}

We argue that there is some truth assignment~$\alpha : X \rightarrow \BB{}$
such that for all truth assignments~$\beta : Y \rightarrow \BB{}$
it holds that~$\psi[\alpha \cup \beta]$ is true
if and only if
there is an agenda~$\Phi' \subseteq \Phi'' \subseteq \Phi$
such that for all~$J^{*} \in \Kemeny(\prof{J}|_{\Phi''})$ it holds
that~$L \subseteq J^{*}$.

$(\Rightarrow)$
Suppose that there exists some truth assignment~$\alpha : X \rightarrow \BB{}$
such that for all truth assignments~$\beta : Y \rightarrow \BB{}$
it holds that~$\psi[\alpha \cup \beta]$ is true.
Consider the agenda~$\Phi''$ that is defined by
letting~$\Pre{\Phi''} = \Pre{\Phi'} \cup \SB x_{i} \SM i \in [n], \alpha(x_i) = 1 \SE
\cup \SB x'_{i} \SM i \in [n], \alpha(x_i) = 0 \SE$.
We argue that~$\Kemeny(\prof{J}|_{\Phi''}) = \SBs J^{*}_{\alpha} \SEs$,
where~$J^{*}_{\alpha}$ is defined as described in
Table~\ref{table:control-addition-cautious-hardness-new-winner}.
The cumulative Hamming distance from~$J^{*}_{\alpha}$ to~$\prof{J}|_{\Phi''}$
is~$3m+3+3u$.
Any complete (for the agenda~$\Phi''$) and $\Gamma$-consistent judgment set
that would have less cumulative Hamming distance to~$\prof{J}|_{\Phi''}$
would be of the form~$J^{*}_{\beta}$, for some~$\beta : Y \rightarrow \BB$,
as described in
Table~\ref{table:control-addition-cautious-hardness-new-winner}.
Any such judgment set~$J^{*}_{\beta}$ would have a cumulative Hamming
distance of~$3m+2+3u$ to the profile~$\prof{J}|_{\Phi''}$.
However, since~$\psi[\alpha]$ is valid, we know that~$J^{*}_{\beta}$
is not $\Gamma$-consistent.
Therefore,~$\Kemeny(\prof{J}|_{\Phi''}) = \SBs J^{*}_{\alpha} \SEs$.
Moreover~$L \subseteq J^{*}_{\alpha}$.
Thus, we know that
there exists an agenda~$\Phi' \subseteq \Phi'' \subseteq \Phi$
such that for all~$J^{*} \in \Kemeny(\prof{J}|_{\Phi''})$ it holds
that~$L \subseteq J^{*}$.

\begin{table}[ht!]
  \begin{center}
  \begin{tabular}{c | c @{\ \ } c @{\ \ } c @{\ \ } c @{\ \ } c @{\ \ } c @{\ \ } c @{\ \ } c @{\ \ } c @{\ \ } c @{\ \ } c}
    \toprule
    & $x_{i,j}$ & $x'_{i,j}$ & $y_j$ & $y'_j$ & $t_j$ &
       $a$ & $b$ &
       $u_{1,\ell}$ & $u_{2,\ell}$ & $u_{3,\ell}$ \\
    \midrule
    $J^{*}_{\alpha}$ & $1$ & $1$ & $0$ & $0$ & $1$ &
      $0$ & $1$ & $0$ & $0$ & $0$ \\
    $J^{*}_{\beta}$ & $1$ & $1$ & $\beta(y_j)$ & $1-\beta(y_j)$ & $0$ &
      $1$ & $0$ & $0$ & $0$ & $0$ \\
    \bottomrule
  \end{tabular}
  \end{center}
  \caption{The judgment sets~$J^{*}_{\alpha},J^{*}_{\beta}$
    for the agenda~$\Phi''$ that is used in the
    proof of Proposition~\ref{prop:control-addition-cautious-hardness}.}
  \label{table:control-addition-cautious-hardness-new-winner}
\end{table}

$(\Leftarrow)$
Conversely, suppose that
there exists an agenda~$\Phi' \subseteq \Phi'' \subseteq \Phi$
such that for all~$J^{*} \in \Kemeny(\prof{J}|_{\Phi''})$ it holds
that~$L \subseteq J^{*}$.
By construction of~$\Gamma$, we know that~$\Phi''$ must
contain exactly one of~$x_i$ and~$x'_i$, for each~$i \in [n]$.
If this were not the case, one could take some
set~$J^{*} \in \Kemeny(\prof{J}|_{\Phi''})$
and construct the set~$J' = J^{*} \backslash \SBs b, \neg a \SEs
\cup \SBs a, \neg b \SEs$.
This judgment set~$J'$ has a strictly smaller cumulative Hamming
distance to the profile~$\prof{J}|_{\Phi''}$,
and is complete and $\Gamma$-consistent.
This is a contradiction with the fact
that~$J^{*} \in \Kemeny(\prof{J}|_{\Phi''})$.

Now consider the truth assignment~$\alpha : X \rightarrow \BB$
that is defined by letting~$\alpha(x_i) = 1$ if and only
if~$x_i \in \Phi''$.
We claim that~$\psi[\alpha]$ is valid, i.e., that for all truth
assignments~$\beta : Y \rightarrow \BB$ it holds that~$\psi[\alpha \cup \beta]$
is true.
Suppose that there exists a truth assignment~$\beta : Y \rightarrow \BB{}$
such that~$\psi[\alpha \cup \beta]$ is false.
Moreover, consider the judgment set~$J^{*}_{\beta}$
that is defined as described in
Table~\ref{table:control-addition-cautious-hardness-new-winner}.
This judgment set has a strictly smaller cumulative Hamming distance
to the profile~$\prof{J}|_{\Phi''}$ than any
set~$J^{*} \in \Kemeny(\prof{J}|_{\Phi''})$.
Moreover, since~$\psi[\alpha \cup \beta]$ is false, we know that~$J^{*}_{\beta}$
is $\Gamma$-consistent.
This is a contradiction, and thus we can conclude that there
exists no truth assignment~$\beta : Y \rightarrow \BB{}$
such that~$\psi[\alpha \cup \beta]$ is false.
In other words,~$\psi[\alpha]$ is valid.
\end{proof}

\begin{corollary}
\label{cor:control-addition-brave-hardness}
\BraveExactControlByAddingIssues{\Kemeny}
is \SigmaP{2}-hard.
\end{corollary}
\begin{proof}
\SigmaP{2}-hardness for the problem
\BraveExactControlByAddingIssues{\Kemeny},
follows from the proof of Proposition~\ref{prop:control-addition-cautious-hardness}.
In this proof, either
all Kemeny outcomes for the profile~$\prof{J}|_{\Phi''}$
include the set~$L$ or
no Kemeny outcomes for the profile~$\prof{J}|_{\Phi''}$
include the set~$L$.
Therefore, the reduction can also be used to show \SigmaP{2}-hardness
for \BraveExactControlByAddingIssues{\Kemeny}.
\end{proof}

\begin{corollary}
\label{cor:control-deletion-membership}
\CautiousExactControlByDeletingIssues{\Kemeny}
and \BraveExactControlByDeletingIssues{\Kemeny}
are in \SigmaP{2}.
\end{corollary}
\begin{proof}
The algorithm described in the proof of
Proposition~\ref{prop:control-addition-cautious-membership}
can straightforwardly be modified to form nondeterministic
polynomial-time algorithms with access to an \NP{} oracle
that solve \CautiousExactControlByDeletingIssues{\Kemeny}
and \BraveExactControlByDeletingIssues{\Kemeny}.
\end{proof}

\begin{corollary}
\label{cor:control-deletion-hardness}
\CautiousExactControlByDeletingIssues{\Kemeny}
and \BraveExactControlByDeletingIssues{\Kemeny}
are \SigmaP{2}-hard,
even for the case where~$\Phi' = \emptyset$.
\end{corollary}
\begin{proof}
\SigmaP{2}-hardness for \CautiousExactControlByDeletingIssues{\Kemeny}
and \BraveExactControlByDeletingIssues{\Kemeny}
for the case where~$\Phi' = \emptyset$
follows from the proof of
Proposition~\ref{prop:control-addition-cautious-hardness}.
In the setting used in this hardness proof, the only way to
achieve Kemeny outcomes~$J^{*}$ with~$L \subseteq J^{*}$
is to choose agendas~$\Phi''$ of the form described in the proof.
Therefore, the \SigmaP{2}-hardness proof works even for the case
where~$\Phi' = \emptyset$.
\end{proof}

\subsection{Constraint-Based Judgment Aggregation}

The scenario of control by adding or deleting issues can also be
formulated in the constraint-based judgment aggregation framework.
In this case, it makes most sense to consider the extended variant
of the constraint-based judgment aggregation framework where
the integrity constraint~$\Gamma$ is allowed to contain propositional
variables that are not contained in the set~$\III$ of issues.
Because in the scenario of control by adding or deleting issues,
the integrity constraint is fixed, whereas the set of issues is subject
to control by the external agent, it is unclear how to adapt the
integrity constraint after the set of issues has been changed.
In the extended constraint-based judgment aggregation framework,
there is no need to adapt the integrity constraint after the set of
issues has been changed.

The result of Theorem~\ref{thm:control-theorem} also holds in the
extended variant of the constraint-based judgment aggregation framework.
The membership results of
Proposition~\ref{prop:control-addition-cautious-membership}
and Corollaries~\ref{cor:control-addition-brave-membership}
and~\ref{cor:control-deletion-membership} also apply to this setting.
Moreover, the proofs of the hardness results of
Proposition~\ref{prop:control-addition-cautious-membership}
and Corollaries~\ref{cor:control-addition-brave-hardness}
and~\ref{cor:control-deletion-hardness} involve agendas containing
only propositional variables and their negations.
Therefore, these hardness results also work for the setting
of constraint-based judgment aggregation.

\section{Conclusion}
\label{sec:conclusion}

We investigated several decision problems that formalize the problems
of manipulation, bribery and control (by adding or deleting issues)
for the Kemeny judgment aggregation procedure.
We showed that these problems are \SigmaP{2}-complete in their
general formulation.
The intractability results that we developed in this paper open up a
wide range of natural questions for future research.

It would be interesting to study to what extent these
intractability results for strategic behavior for the Kemeny rule
hold up in a more refined computational complexity analysis.
That is, do these intractability results also hold (1)~when we consider
restricted logic languages to specify the relation between issues
and (2)~when we use the more refined framework of parameterized
complexity to capture natural restricted settings of the various
problems.
In order to compare different judgment aggregation
procedures on the basis of the computational complexity of
strategic behavior, it is necessary to also study these properties
for further judgment aggregation procedures.
Another natural direction for future research is to study variants
of strategic behavior that are based on different notions of
preference over judgment sets that have been considered
in the literature.
Finally, the types of strategic behavior that we considered in this
paper are not the only relevant types that one needs to consider.
An example of strategic behavior that we did not consider in this
paper is control by adding or removing individuals (rather
than adding or removing issues).
It would be interesting to study such additional notions of
strategic behavior in judgment aggregation from a computational
complexity point of view as well.

%




\DeclareRobustCommand{\DE}[3]{#3}

\bibliographystyle{abbrv}



\begin{contact}
Ronald de Haan\\
Algorithms \& Complexity Group\\
Technische Universit\"{a}t Wien\\
Vienna, Austria\\
\email{dehaan@ac.tuwien.ac.at}
\end{contact}


\end{document}